\documentclass{article}

\PassOptionsToPackage{numbers,compress}{natbib}

\usepackage[main,preprint]{style/neurips_2026}

\usepackage[utf8]{inputenc} 
\usepackage[T1]{fontenc}    
\usepackage{hyperref}       
\usepackage{url}            
\usepackage{booktabs}       
\usepackage{amsfonts}       
\usepackage{nicefrac}       
\usepackage{microtype}      
\usepackage{xcolor}         

\usepackage{subcaption}
\usepackage[export]{adjustbox}

\newcommand\supp{\text{supp}}
\newcommand\opt{\text{OPT}}

\newcommand\src{\nu_{\text{src}}}
\newcommand\tgt{\nu_{\text{tgt}}}
\newcommand\dSrc{\rho_{\text{src}}}
\newcommand\dTgt{\rho_{\text{tgt}}}
\newcommand\Wc{\mathcal{W}}
\newcommand\fwd{\mathcal{F}}
\newcommand\bwd{\mathcal{B}}
\newcommand\Sc{\mathcal{S}}
\newcommand\norm[1]{\left\|#1\right\|}
\newcommand\Lc{\mathcal{L}}
\newcommand\hLc{\wh{\Lc}}

\usepackage{etoc}
\usepackage{microtype}
\usepackage{graphicx}
\usepackage[T1]{fontenc}
\usepackage[english]{babel}
\usepackage[utf8]{inputenc}
\usepackage{booktabs}
\usepackage{array}
\usepackage{amsmath}
\usepackage{amssymb}
\usepackage{mathtools}
\usepackage{amsthm}
\usepackage{bm}
\usepackage{thmtools}
\usepackage{thm-restate}
\usepackage{hyperref}
\hypersetup{
    colorlinks = true,
    linkcolor = blue,
    citecolor=blue,
    urlcolor=black,
}
\usepackage{xcolor}
\definecolor{mblue}{rgb}{0.0,0.0,1.0}
\definecolor{mgreen}{rgb}{0.0,0.501960784314,0.0}
\definecolor{mred}{rgb}{1.0,0.0,0.0}
\usepackage[capitalize]{cleveref}
\usepackage[ruled]{algorithm2e}
\usepackage[noend]{algpseudocode}

\usepackage{bbm}
\usepackage{xcolor}
\usepackage{parskip}
\usepackage{setspace}
\usepackage{comment}
\usepackage{caption}
\usepackage{enumitem}
\usepackage{wrapfig}

\newtheorem{theorem}{Theorem}[section]

\newtheorem{corollary}{Corollary}[section]

\newtheorem{lemma}{Lemma}[section]

\newcommand\R{\mathbb{R}}

\newcommand\Cc{\mathcal{C}}

\newcommand\Mc{\mathcal{M}}
\newcommand\Tc{\mathcal{T}}
\DeclareMathOperator*{\E}{\mathbb{E}}

\newcommand{\br}[1]{\left({#1}\right)}
\newcommand{\bs}[1]{\left[{#1}\right]}
\newcommand{\abs}[1]{\left| {#1} \right|}
\newcounter{protocol}
\makeatletter

\makeatletter
\providecommand*{\bigcupdot}{%
  \mathop{%
    \vphantom{\bigcup}%
    \mathpalette\@bigcupdot{}%
  }%
}
\newcommand*{\@bigcupdot}[2]{%
  \ooalign{%
    $\m@th#1\bigcup$\cr
    \sbox0{$#1\bigcup$}%
    \dimen@=\ht0 %
    \advance\dimen@ by -\dp0 %
    \sbox0{\scalebox{1.5}{$\m@th#1\cdot$}}%
    \advance\dimen@ by -\ht0 %
    \dimen@=.5\dimen@
    \hidewidth\raise\dimen@\box0\hidewidth
  }%
}
\makeatother

\newcommand{\Xsrc}{X_{\text{src}}}
\newcommand{\Xtgt}{X_{\text{tgt}}}

\newcommand{\kprimal}{K}
\newcommand{\dirac}[1]{\delta_{#1}}
\newcommand{\ustar}{u^\star}

\newcommand{\psistar}{\psi^\star}
\newcommand{\Vstar}{V^\star}

\newcommand{\pistar}{\pi^\star}
\newcommand{\pibar}{\overline{\pi}}
\newcommand{\mustar}{\mu^\star}
\newcommand{\hmu}{\wh{\mu}}

\newcommand{\Htest}{H_{\text{test}}}

\newcommand{\mupi}{\mu^\pi}

\newcommand{\dd}{\mathrm{d}}

\newcommand{\X}{\mathcal{X}}
\newcommand{\U}{\mathcal{U}}

\newcommand{\LL}{\mathcal{L}}
\newcommand{\GG}{\mathcal{G}}
\newcommand{\hGG}{\wh{\GG}}

\newcommand{\real}{\mathbb{R}}

\newcommand{\EE}[1]{\mathbb{E}\left[#1\right]}

\newcommand{\EEpi}[1]{\mathbb{E}_{\pi}\left[#1\right]}
\newcommand{\PPpi}[1]{\mathbb{P}_{\pi}\left[#1\right]}
\newcommand{\EEs}[2]{\mathbb{E}_{#2}\left[#1\right]}

\newcommand{\ra}{\rightarrow}

\newcommand{\ev}[1]{\left\{#1\right\}}

\newcommand{\pa}[1]{\left(#1\right)}
\newcommand{\bpa}[1]{\bigl(#1\bigr)}

\newcommand{\wh}{\widehat}
\newcommand{\wt}{\widetilde}

\usepackage{todonotes}
\definecolor{PalePurp}{rgb}{0.66,0.57,0.66}

\newcommand{\Law}{\text{Law}}

\usepackage[normalem]{ulem}
\usepackage{enumitem}

\title{Generative Modeling by Value-Driven Transport}

\author{%
  Pablo Moreno-Mu\~noz \thanks{Authors listed alphabetically.} \\
  \small{Universitat Pompeu Fabra}\\
  \small{Barcelona, Spain} \\
  \small{\texttt{pablo.moreno@upf.edu}} \\
  \And
  Adrian M\"uller $^*$\\
  \small{ETH Z\"urich}\\
  \small{Zürich, Switzerland}\\
  \small{\texttt{adrian.mueller@inf.ethz.ch}}\\
  \And
  Gergely Neu $^*$\\
  \small{ICREA \& Universitat Pompeu Fabra}\\
  \small{Barcelona, Spain}\\
  \small{\texttt{gergely.neu@gmail.com}}\\
}

\begin{document}

\maketitle

\begin{abstract}
  We propose a new framework for generative modeling based on a discrete-time stochastic control formulation of measure 
transport. Adapting classic results from control theory, we formulate our problem as a linear program whose dual 
variables correspond to the \emph{optimal value function} of the control problem, which directly encodes the optimal 
control policy. Exploiting this LP formulation, we develop an efficient simulation-free primal-dual algorithm for 
computing approximately optimal value functions and the associated \emph{value-driven transport} (VDT) policies which 
approximate the true optimal policy. We show that well-trained VDT policies enjoy numerous favorable properties in 
comparison with other state-of-the-art methods based on flows, diffusions, or Schr\"odinger bridges: they lead to 
straight transport paths which can be simulated quickly and robustly, and can be enhanced in all the same ways as 
diffusion and flow-based models (e.g., conditional generation, classifier-free guidance, unpaired data-to-data 
translation are all easy to incorporate). We evaluate our methodology in a range of experiments, 
with results that indicate strong performance and good potential for scalability. 

\end{abstract}

\section{Introduction}
Many modern generative modeling tasks can be naturally formulated as problems of measure transport, whereby one seeks a 
transformation that turns samples from a given source distribution into samples from a desired target distribution. The rich literature on optimal transport (OT, \citep{Vil03,San15,peyre2025optimal}) provides solid mathematical 
foundations for formulating such measure transport problems and 
characterizing their solutions, making it a natural starting point for developing computational methods for generative 
modeling. 
In this work, we introduce a new algorithmic framework that approximates optimal transport plans by combining tools 
from modern deep learning and optimal control.

The core of our methodology is a reformulation of the classic \emph{dynamic optimal transport} problem of 
\citet{benamou2000computational} as a discrete-time stochastic control problem, where an initial state distribution is 
steered by a control policy in a fixed number of $H$ decision rounds, aiming to hit the target distribution at the 
final round while minimizing movement cost along the way. Using classic tools from control theory, we show that 
optimal control policies can be written in terms of a \emph{value function}, with the evolution of the states 
determined by the gradient of the value functions. Motivated by this observation, we consider the class of 
\emph{value-driven transport} policies that push states along the gradient fields of approximate value functions, 
and develop an algorithmic framework for computing the values. For this purpose, we adapt a classic linear-programming 
(LP) formulation of optimal control to our needs, where the optimal dual variables correspond to the optimal value 
functions.

Our algorithmic contribution is the development of a new primal-dual method for approximately solving this LP---or, 
equivalently, to find a saddle-point of its Lagrangian. 
At a high level, our method is based on approximating the associated dual function by numerically optimizing the primal 
variables and evaluating stochastic gradients with respect to the dual variables using small to moderate minibatches of 
data. Both primal and dual updates can be performed efficiently, without needing access to simulated trajectories 
from the learned model. The key to this efficiency is parametrizing the primal variables as particle clouds initialized 
as samples of a stochastic interpolant between the source and target distributions, which are then updated using 
Wasserstein gradient descent. With this choice, implementing our method only requires a single neural network, used for 
parametrizing the value functions. We verify the effectiveness of this method in an exhaustive range of 
experiments, with the results 
indicating good robustness and scalability properties of the training procedure and convincing performance of the 
learned transport policies.

Our approach has a number of conceptual and practical advantages over competing methods for generative modeling. The 
most notable of these are consequences of a well-known valuable property of the optimal solution to the dynamic OT
problem: the solution paths generated by the optimal policy are \emph{linear interpolations} of the (appropriately 
coupled) source and target points. This implies that the optimal policy could in principle be used for generating 
samples from the target using \emph{just a single step} (since under an optimal policy, all consecutive steps would be 
done in the same direction, with the same stepsize). While in practice, the level of precision needed for such one-step 
generation is hard to achieve due to errors of approximation and optimization, our empirical results show that the 
straight-path property of the optimal policy can indeed be used to radically reduce the number of generation steps at 
test time, up to a factor of 10 in comparison to the horizon $H$ used at training time. Another 
appealing property of 
our method is that it learns the value function as opposed to attempting to directly learn the optimal policy, 
whose value-driven structure would be challenging to enforce otherwise when working with parametrized approximations. 
Additionally, since our problem formulation is symmetric with respect to the roles source and target distributions, the 
learned value functions can be used equally effectively for transport from source to target and vice versa.

The techniques used for deriving and implementing our method are all quite different from the standard tools of the 
generative modeling literature. Instead of using concepts from continuous-time mathematics (SDEs and their time 
reversals) and denoising (score functions and their estimation) like nearly all state-of-the-art methods for generative 
modeling, our approach is rooted in discrete-time stochastic optimal control---a setting more commonly considered under 
the umbrella of reinforcement learning (RL, \citep{SB18}). Indeed, our problem formulation can be seen to be equivalent 
to finding optimal policies in a constrained Markov decision process (MDP, \citep{Puterman1994,Ber07:DPbookVol1}), with 
the objective of generating from the target being formulated in terms of a constraint. Due to this constraint, our 
problem cannot be directly addressed with standard 
tools of RL which are designed for optimizing a single numerical reward function without constraints. Additionally, 
common RL methods are developed with unknown state dynamics in mind, which makes them unlikely to perform well in our 
setting where the state dynamics is perfectly known and controllable without restrictions. 
That said, our algorithm design draws quite a bit of inspiration from the RL literature, particularly from works that 
use LP formulations as their starting point \citep{schwesei85,FaVaRo03,BCKN21,LMMN21,NO25}. In these works, the LP 
framework has largely used as a tool for theoretical analysis, and the methods themselves have arguably achieved only 
limited empirical success so far. Thus, our encouraging experimental results might be seen as a validation that the LP 
framework can indeed serve as the basis of competitive large-scale learning algorithms. We believe that the techniques 
we develop here might prove useful in addressing other control tasks in areas such as RL as well. 

The rest of the paper is organized as follows. After providing a quick primer on optimal transport in 
Section~\ref{sec:prelim}, we present our framework and algorithm in Section~\ref{sec:VDT}, and then 
Section~\ref{sec:experiments} describes our experiments. The main text is concluded in Section~\ref{sec:conc}.

\paragraph{Notation.} We denote the set of (Radon) probability measures on $\real^{d\times d}$ by $\Sc$. For a probability distribution $p\in\Sc$ and a function $f:\real^{d\times d}\ra \real^d$, we use $f_{\#}p$ to denote the pushforward of $p$ under $f$ (i.e., the distribution of $f(X)$ with $X\sim p$, and we extend this definition to arbitrary domains and co-domains of $f$). We define the maps 
$\fwd, \bwd \colon \R^d \times \R^d \to \R^d$ as the projections $\bwd(x,y)=x$ and $\fwd(x,y)=y$, which respectively induce the marginalization operators $\bwd_{\#}$ and $\fwd_{\#}$ acting on joint distributions. 
Throughout the paper, equalities between functions are to be understood to hold almost everywhere (a.e.), and 
minimization over functions as minimization over continuous functions. 

\section{Preliminaries}\label{sec:prelim}
We consider the problem of optimal transport (OT) over $\real^d$ with the squared Euclidean distance $\frac 12 
\norm{x - y}^2$ as transport cost. This is a very well-studied problem with countless equivalent formulations and 
connections with many areas of mathematics, probability, and physics \citep{Vil03,San15}. Here, we state a few 
equivalent definitions and basic results that are directly relevant to our own formulation of the problem, which we 
state in the second half of this section. We refer to the books of \citet{Vil03} and \citet{San15} for an exhaustive 
treatment of the subject, and recommend the notes of \citet{peyre2025optimal,Pey25} as a more accessible quick 
introduction.

\begin{figure}
\centering
\begin{minipage}{.45\textwidth}
 \includegraphics[width=\textwidth]{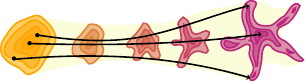}
 \\
 $\nu_0 = \src \quad\,\, \dd X_t = u_t(X_t) \dd t \quad\,\, \nu_1 = \tgt$
 \vspace{.1cm}
 \\
 \centering
\end{minipage}
\hspace{1cm}
\begin{minipage}{.45\textwidth}
 \includegraphics[width=\textwidth]{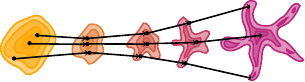}
 \\
 $\nu_0 = \src \quad\,\, X_{h+1} = \pi_h(X_h) \quad\,\, \nu_H = \tgt$
 \vspace{.1cm}
 \\
 \centering
\end{minipage}
\caption{Continuous- and discrete-time transport plans. Continuous-time plans are defined as smooth solutions of ODEs 
parametrized by a continuous-time control drift, whereas discrete-time plans are defined by a sequence of 
control policies directly mapping states to next states.}
\end{figure}

Let us consider two measures $\src$ and $\tgt$ on $\real^d$, to be referred to as the \emph{source} and \emph{target} 
distributions, respectively. We will often refer to $\X = \real^d$ as the \emph{state space} and elements $x\in\X$ as 
\emph{states}. We assume that both distributions admit densities with respect to the Lebesgue measure, and aim to 
solve the so-called 
\emph{Monge problem}\footnote{The factor $\frac 12$ featured in all our definitions is not common in the 
literature, but will be helpful for our derivations.}:
\begin{align}
    \Wc^2_2(\src,\tgt) = \min_{\substack{M \colon \R^d \to \R^d \colon  \\ M_{\#}\src = \tgt}} \frac 12 \int \|M(x) - 
x\|^2 d\src(x).
\label{eq:monge}
\end{align}
This problem is known to admit a unique solution under our conditions specified above, and the value of the 
minimum is called the \emph{squared Wasserstein-2 distance} between $\tgt$ and $\src$. The distance can 
equivalently be shown to be the value of the \textit{Kantorovich problem} 
\begin{align}
    \Wc^2_2(\src,\tgt) = \min_{\gamma \in \Gamma(\src,\tgt)} \frac 12 \int \|x-y\|^2 d\gamma(x,y), 
\label{eq:kantorovich}
\end{align}
where the domain of optimization is the set of all couplings of $\src$ and $\tgt$: $\Gamma(\src, \tgt) = \{\gamma \in 
\Sc \mid \bwd_{\#}\gamma 
= \src \wedge \fwd_{\#}\gamma = \tgt \}$.
The minimizer is unique and satisfies $\gamma^\star = (\text{id}, M^\star)_{\#} \src$.

Another equivalent formulation due to \citet{benamou2000computational} poses the OT problem as a problem of stochastic 
control. This framework sets up a control problem in which an initial state $X_0 \in \real^d$ is drawn from the 
source distribution $\src$ at the initial time instant $t=0$, and the continuous process of states is then steered by a 
time-dependent \emph{control drift}  $u_t:\real^d \ra \real^d$ according to the continuous-time dynamics $\dd X_t 
= u_t(X_t) \dd t$. We call the control \emph{feasible} if the distribution of the resulting final state $X_1$ matches 
the 
target at time $t=1$ (i.e., $\Law(X_1) = \tgt$), and denote all such controls by $\U\pa{\src,\tgt}$. Denoting the 
expectation operator under the path distribution induced by $u$ as $\EEs{\cdot}{u}$, the optimal transport cost can be 
equivalently written as
\begin{align}
    \Wc^2_2(\src,\tgt) = \min_{u \in \U(\src,\tgt)} \frac 12 \int_{[0,1]} \EEs{\norm{u_t(X_t)}^2}{u} \dd t 
\label{eq:benamou-brenier}.
\end{align}
The optimal control can be shown to satisfy $\ustar_t(x) = \nabla_x \psistar_t(x)$ for some scalar \emph{potential 
function} $\psistar:[0,1]\times\real^d \ra \real$.  Thus, solving the continuous-time optimal control 
problem reduces to finding the potential $\psistar$. Another fact which will be important for us is that the 
optimal control paths are linear interpolations of the optimally coupled points $(X_0,X_1)\sim \gamma^\star$, with $X_t 
= (1-t) X_0 + t X_1$ satisfied for all $t\in[0,1]$.

Our approach is based on a discrete-time analogue of this latter formulation of the OT problem, where a sequence of 
states are generated recursively over a sequence of discrete \emph{time steps} $h=0,1,\dots,H$ (where $H$ is called the 
\emph{horizon}). 
The initial state $X_0$ is drawn from the source distribution $\src$, and each consecutive state 
$X_0,X_1,\dots,X_{H}$ is selected according to a \emph{control policy} $\pi$ that maps each state $X_h$ to the next 
state 
as $X_{h+1} = \pi_h(X_h)$. The control cost associated with each state transition is $c(X_h,X_{h+1}) = \frac{H+1}{2} 
\norm{X_h - X_{h+1}}^2$. In general, we will consider non-stationary policies represented as the collection of 
deterministic maps $\pi = \ev{\pi_h:\real^d \ra \real^d}_{h=0}^H$.
A policy is called \emph{feasible} if the law of the final state $X_{H+1}$ generated by the above process matches the 
target distribution ($\Law(X_{H+1}) = \tgt$), and we denote the set of feasible control policies by $\Pi(\src,\tgt)$. 
Letting $\EEpi{\cdot}$ stand for the expectation operator under the joint distribution of state sequences generated by 
$\pi$, we then set up our objective as
\begin{align}
    \Wc^2_2(\src,\tgt) = \min_{\pi \in \Pi(\src,\tgt)} \frac{H+1}{2} \sum_{h=0}^{H} \EEpi{\norm{X_{h+1} - X_h}^2} 
\label{eq:DDOT}.
\end{align}
In the following section, we will show that this optimization problem is well-posed, has a unique optimal solution, and 
its value is equivalent to all the above definitions of the squared Wasserstein-2 distance (as our notation already 
suggests).

\section{Value-driven transport}\label{sec:VDT}
There are many possible ways to address the discrete-time dynamic OT problem of Equation~\eqref{eq:DDOT}. 
Unfortunately, since our problem features a hard constraint
on the terminal-state distribution, the standard theory of dynamic programming does not 
apply \citep{bellman57,howard60,Ber07:DPbookVol1}. 
Instead, our approach is based on the \emph{linear programming} formulation of discrete-time optimal control, 
originally developed in a series of papers in the 1960s \citep{Man60,Ghe60,dEp63,Den70}:
we restate our problem as optimization over the space of state distributions that can be generated by 
feasible control policies, which space is exactly characterized by a set of linear constraints. 
In Section~\ref{sec:LP} below, we formalize this optimization problem and state several fundamental claims about its 
solution. Due to space constraints, we only provide the essentials here, and defer the complete statements along with 
their derivations to Appendix~\ref{app:lp-properties}.
Given these foundations, Section~\ref{sec:VDT_alg} then describes our main algorithmic 
contribution: a scalable stochastic optimization method for computing near-optimal solutions of the LP.

\subsection{Discrete-time dynamic OT as a linear program}\label{sec:LP}
A key concept of our framework is that of \emph{occupancy measures}. The occupancy measure associated with a control 
policy $\pi$ is the collection of joint distributions $\mupi = \ev{\mupi_h \in \Delta_{\real^d \times 
\real^d}}_{h=0}^{H}$ over the state space, with $\mupi_h = \PPpi{(X_h,X_{h+1}) \in \cdot}$ corresponding to the 
distribution of the state pair $X_h,X_{h+1}$ generated by policy $\pi$ under the discrete-time control process 
described in Section~\ref{sec:prelim}. As 
is well-known in the context of Markov decision processes, the set of all valid occupancy measures that can be induced 
by control policies is uniquely characterized by a set of linear constraints (see, e.g., Chapter~6.9 in 
\citet{Puterman1994}). Adapting these results to our needs, we formulate our optimal control problem as the following 
linear program:
\begin{equation}
\begin{split}
    \Wc^2_{\text{dyn},2}(\src,\tgt) = \min_{\mu \in \Sc^{H+1}} \frac{H+1}{2}&\sum_{h=0}^{H} \int \|x-y\|^2 
d\mu_h(x,y) \\ 
    \text{s.t.} \quad\quad \src &= \bwd_{\#} \mu_0\\
    \qquad\;\fwd_{\#} \mu_h &= \bwd_{\#} \mu_{h+1} \qquad (\forall h \in \{0,\dots,H-1\})\\
    \qquad\,\fwd_{\#} \mu_H &= \tgt 
    \label{eq:primal-lp}
\end{split}
\tag{LP}
\end{equation}
The first and last constraints are called \emph{source} and \emph{target} constraints, and the intermediate ones are 
called \emph{flow constraints}. Intuitively, these express the criterion that the measures $\mu_h$ should be
temporally consistent across different time indices $h$.
The single-stage version of the problem ($H=0$) is clearly equivalent to the 
Kantorovich problem~\eqref{eq:kantorovich}. Our first technical result shows that this remains true for arbitrary 
choices of $H$.
\begin{theorem} \label{thm:DOT-structure}
The solution $\mustar$ of \eqref{eq:primal-lp} satisfies
$\Wc_{\text{dyn},2}^2(\src,\tgt) = 
\Wc_2^2(\src,\tgt)$. Moreover:
\begin{enumerate}[leftmargin=*,itemsep=0pt,label=\alph*.]
 \item Let $\gamma^\star$ be optimal for the Kantorovich problem~\eqref{eq:kantorovich}, define 
$(X_0,X_{H+1})\sim \gamma^\star$ and let $X_h = X_0 + \frac{h}{H+1} (X_{H+1}-X_0)$. Then the laws $\mu_h^\star$ of 
$(X_h,X_{h+1})$ form an optimal solution $\mu^\star$ for \eqref{eq:primal-lp}. 
 \item Suppose that the static Monge problem~\eqref{eq:monge} admits a solution $M^\star$. Then the discrete-time 
dynamic OT problem~\eqref{eq:DDOT} admits an optimal solution $\pi^\star$.
Furthermore, $\mu^\star$ given by $\mu^\star_h = (\pistar_{h-1}\circ\cdots\circ \pistar_0 ,\pistar_{h}\circ\cdots\circ 
\pistar_0)_{\#} \src$ and $\mustar_0 = (\text{id}, \pistar_0)_\# \src$
is an optimal 
solution for \eqref{eq:primal-lp}. 
\end{enumerate}
\end{theorem}
Notice that the two claims together imply that the optimal control policy $\pistar$ consists of moving each source 
point along a straight line towards a designated target point, and $\mustar$ represents the probability distributions 
of the corresponding linearly interpolated states.

\paragraph{Duality.} The Lagrangian function associated with~\eqref{eq:primal-lp} is defined by introducing a set of 
dual variables $V:\real^{d}\ra\real$ to enforce the constraints, and adding them to the objective:
\begin{equation*}\label{eq:lagrangian}
\begin{split}
    \Lc(\mu, V) =& 
    \sum_{h=0}^H \!\int\! \bpa{c(x,y) + V_{h\!+\!1}(y)-V_h(x)} \dd\mu_h(x,y) 
    +\! \int \! V_0(x) \dd\src(x) -\! \int \! V_{H\!+\!1}(x)\dd\tgt(x).
\end{split}
\end{equation*}
Then, our problem can be equivalently rewritten as the min-max optimization problem
$\Wc^2_{\text{dyn},2}(\src,\tgt) =  \inf_{\mu\in\Sc^{H+1}} \sup_{V\in\Cc(\Omega)^{H+2}}~ \Lc(\mu,V)$, which can be shown to admit a unique solution $(\mustar,\Vstar)$ on the transported support.\footnote{The solution $\Vstar$ is only unique up to a 
constant additive shift. For technical reasons, here we restrict the problem to the domains $\Mc:=\Mc_{\geq 
0}(\Omega\times\Omega)$ and $\Cc(\Omega)$ over a compact convex set $\Omega\subset\R^d$ rather than all of $\R^d$, by 
assuming that $\supp(\src), \supp(\tgt) \subset \Omega$. All previous statements then hold analogously.}
Additionally, the optimal dual variables 
$\Vstar$ maximize the associated \emph{dual function} defined as
$\GG(V) = \inf_{\mu\in\Sc^{H+1}} ~ \Lc(\mu,V)$.
We show below that $\Vstar$ satisfies a system of equations analogous to the \emph{Bellman optimality 
equations} from the theory of (unconstrained) optimal control and dynamic programming. 
By this analogy, we borrow the terminology of dynamic 
programming and refer to the set of dual variables $V$ as the \emph{value function} and $\Vstar$ as the \emph{optimal 
value function}. Crucially, we also show that the optimal value function directly encodes the optimal control policy 
$\pistar$. These results are stated formally below.
\begin{theorem} \label{thm:bellman}
    Let $\mu^\star$ be an optimal solution for \eqref{eq:primal-lp} let $\Vstar$ be a maximizer of the dual 
function $\GG$. Then, for all $h=0,1,\dots,H$ and for 
$(\pistar_{h-1}\circ\cdots\circ \pistar_0)_{\#} \src$-almost all $x\in\Omega$, the optimal value 
function and the optimal policy respectively satisfy
    \begin{align*}
        \Vstar_h(x) &= \min_{y\in\Omega} \br{\frac{H+1}{2}\norm{x-y}^2 + \Vstar_{h+1}(y)},
        \\
        \pistar_h(x) &= \arg\min_{y\in \Omega}\br{\frac{H+1}{2}\norm{x-y}^2 + \Vstar_{h+1}(y)}.
    \end{align*}
\end{theorem}
Note that, unlike in standard dynamic programming, $\Vstar_{H+1}$ is typically \emph{not} zero, but rather corresponds 
to the optimal dual variables that penalize violations of the target-distribution constraint. In this sense, 
$\Vstar_{H+1}$ may be seen as a ``discriminator'' function, which can be shown to correspond to the optimal dual 
variables of the Kantorovich problem~\eqref{eq:kantorovich}. 
Even more importantly, the theorem shows that the optimal policy is \emph{greedy} with respect to $\Vstar$. Generally, 
a policy $\pi$ is greedy with respect to a value function $V$ if it maps state $x_{h}$ to $x_{h+1}$ that satisfies
\[
x_{h+1} = \arg\min_{y\in \Omega}\br{\frac{H+1}{2}\norm{x_h-y}^2 + V_{h+1}(y)} =  
x_h - \frac{1}{H+1}\nabla_x V_{h+1}(x_{h+1}).
\]
Note that this system of equations is not easy to solve, as it features $x_{h+1}$ on both the left- and right-hand 
side. We address this issue below by making one more important observation about the structure of the optimal solution 
paths.

\paragraph{Value-driven transport policies.} 
As established in Theorem~\ref{thm:DOT-structure}, the optimal solution paths linearly interpolate between the target 
and source points. By combining this insight with the above fact about $\pistar$, we deduce that states sampled from 
the optimal policy evolve as
\begin{align}\label{eq:gradient_dir}
 X_{h+1} &= X_h + \frac{1}{H+1} (X_{H+1}-X_0)
 = X_h - \frac{1}{H+1} \nabla_x \Vstar_{h+1}(X_{h+1}),
\end{align}
which implies that $\nabla_x \Vstar_{h}(X_{h}) = (X_0 - X_{H+1})$ holds for all $h$. Therefore, the optimal policy can 
be written explicitly as $ \pistar_h(x) = x - \frac{1}{H+1}\nabla_x \Vstar_h(x)$, recovering the form of the 
continuous-time dynamic OT path.
Based on this observation, we define the \emph{value-driven transport (VDT) policy} induced by a generic value function 
$V$ 
as the policy $\pi(\cdot;V)$ mapping states to next states 
as 
\begin{equation}\label{eq:vdt_prediction}
 \pi_h(x;V) = x - \frac{1}{H+1} \nabla_x V_h(x).
\end{equation}
Clearly, we have $\pi(\cdot;\Vstar) = \pistar$, but note that value-driven and greedy policies do not generally 
coincide for value functions other than $\Vstar$. In contrast to greedy policies, VDT policies have the distinct 
advantage of being explicitly computable from $\Vstar$, while remaining a sufficiently expressive class to include the 
optimal policy. Therefore, we will restrict our attention to VDT policies in our algorithm design, which we turn to 
describing below.
\begin{figure}
 \includegraphics[width=\textwidth]{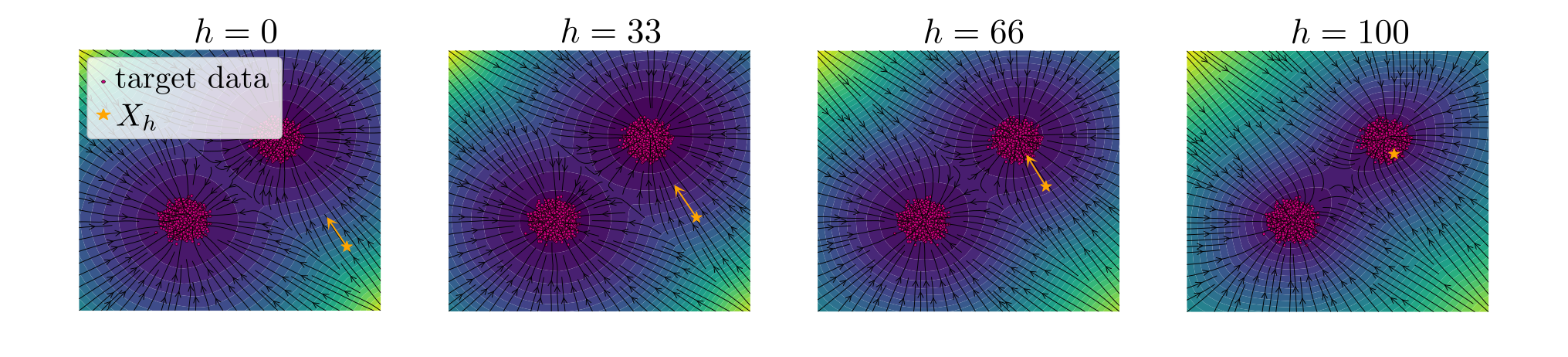}
 \caption{Value functions and value-driven transport policies in a two-dimensional example, plotted for time indices 
$h\in\ev{0,33,66,100}$. A (straight) sample path followed by a source sample and the update directions highlighted in 
yellow.}
\end{figure}

\subsection{Primal-dual methods for VDT training}\label{sec:VDT_alg}
We are now ready to describe our main algorithmic contribution: a primal-dual algorithm for computing near-optimal 
VDT policies. In short, we aim at approximately solving the min-max problem~$\inf_\mu \sup_V 
\Lc(\mu,V)$, and returning the VDT policy induced by the estimated value function. The key practical challenge 
is picking an appropriate parametrization for the primal and dual variables that allows efficient training 
while having only sample access to the distributions $\src$ and $\tgt$.

\paragraph{Dual parametrization and update rule.} We parametrize the dual variables using a neural 
network $V_\theta$ that takes as input a state $x$ and a time index $h$, and returns $V_\theta(x,h)\in\real$. Our ideal 
goal is to perform stochastic gradient ascent on the dual function $\GG$ with respect to the parameters $\theta$. In 
order to construct a sample-based estimator of the required gradients, we consider a minibatch of (possibly 
coupled) samples $(\Xsrc(1),\dots,\Xsrc(b)) \sim \src$ and $(\Xtgt(1),\dots,\Xtgt{(b)}) \sim \tgt$ drawn respectively from the source and target distributions, and use these to define the \emph{empirical Lagrangian}
\begin{align*}
    \hLc(\mu, V) =&     \sum_{h=0}^H \!\int\! \bpa{c(x,y) + V_{h\!+\!1}(y)-V_h(x)} \dd\mu_h(x,y)  + \frac 1b 
\sum_{i=1}^b 
\bpa{V_0\pa{\Xsrc\pa{i}} - V_{H+1}\pa{\Xtgt\pa{i}}}.
\end{align*}
Analogously, we define the \emph{empirical dual function} as $\hGG(V) = \inf_{\mu\in\Sc^{H+1}} ~ \hLc(\mu,V)$. Clearly, both of these are unbiased estimators of their population counterparts $\LL$ and $\GG$, and thus stochastic 
gradients of $\GG$ can be obtained by evaluating the gradients of $\hGG$.

\paragraph{Primal parametrization and update rule.} Computing the empirical dual function $\hGG$ is a highly nontrivial 
task, given that it involves optimizing over the set of probability 
distributions $\Sc^{H+1}$. We propose to address this problem by a 
nonparametric approach known as \emph{Wasserstein gradient descent} \citep{AGS05,JKO98,SKL20,CKK24}. 
Specifically, for each time index $h = 0,1,\dots,H$, we set up a system of $2 b$ \emph{particles} to represent $\mu_h$, 
with a pair of particles corresponding to each sample pair $(\Xsrc(i),\Xtgt(i))$ within the minibatch. Denoting the pair by 
$(X_h^-(i),X_h^{+}(i))$ corresponding to time $h$ and particle $i$, the distribution $\hmu_h$ is written as
$
 \hmu_h = \frac{1}{b} \sum_{i=1}^b \dirac{\pa{X_h^-(i),X_h^+(i)}}
$
where $\delta_z$ is the Dirac measure centered at $z\in\real^{2d}$. 
Then, the locations of the particles are updated by gradient descent as
\begin{equation}\label{eq:particle_update}
\begin{split}
 X_{h}^-(i) &\leftarrow X_{h}^-(i) - \eta \pa{\nabla_x c(X_{h}^-(i),X_{h}^+(i)) - \nabla_x V_{h}(X_{h}^-(i))}
 \\
 X_{h}^+(i) &\leftarrow X_{h}^+(i) - \eta \pa{\nabla_y c(X_{h}^-(i),X_{h}^+(i)) + \nabla_x V_{h+1}(X_{h}^+(i))},
\end{split}
\end{equation}
where $\eta > 0$ is a positive stepsize parameter, and $\nabla_x$ and $\nabla_y$ respectively refer to differentiation 
of the cost function with respect to its first and second arguments. After sufficiently many update steps, the procedure 
returns a particle system which can be used to evaluate the gradient of the approximate dual function 
$\hLc(\hmu,V_\theta)$ with respect to $\theta$ as the average of the sample gradients for all $i=1,2,\dots,b$:
\begin{equation}\label{eq:dual_update}
\begin{split} 
\Delta_\theta(i) =  & \pa{\nabla_\theta V_\theta(\Xsrc(i),0) - \nabla_\theta V_\theta(X_{0}^-(i),0) } +
     \pa{\nabla_\theta V_\theta(X_{H}^+(i),H\!+\!1) -      \nabla_\theta V_\theta(\Xtgt(i),H\!+\!1)}
     \\ 
     & + 
     \sum_{h=1}^H \pa{\nabla_\theta V_\theta(X_{h-1}^+(i),h) - 
     \nabla_\theta V_\theta(X_{h}^-(i),h)}.
     \raisetag{.75cm}
\end{split}
\end{equation}

For the primal updates to succeed, it is greatly beneficial to design a good initialization of the particles. 
Inspired by the structure of the optimal solution $\mustar$ established in Theorem~\ref{thm:DOT-structure}, we thus  
first compute a coupling $(\Xsrc^\star(i),\Xtgt^\star(i))_{i=1}^b$ of the source and target points 
in the minibatch, and set up the initial particle positions for all $h=0,1,\dots,H-1$ as 
\begin{equation}\label{eq:particle_init}
\begin{split}
X_{h}^+(i) = X_{h+1}^-(i) &= \frac{H + 1 - h}{H+1} \cdot \Xsrc^\star(i) + \frac{h}{H+1} \cdot \Xtgt^\star(i).
\end{split}
\end{equation}
The edge cases are initialized as $X_0^-(i) = \Xsrc^\star(i)$ and $X_H^+(i) = \Xtgt^\star(i)$. It is easy to see the resulting 
initial particle cloud $\hmu$ satisfies the constraints of~\eqref{eq:primal-lp}. Additionally, if 
the minibatch samples are coupled according to an OT map, Theorem~\ref{thm:DOT-structure} 
implies that $\hmu$ minimizes the empirical primal function $\sup_{V} \hLc(\cdot,V)$. Thus, this initialization is 
expected to be nearly perfect when $V_\theta$ is close to $\Vstar$, and the number of primal updates necessary for good 
performance can be very low. 

\paragraph{The algorithm.} Our algorithm for training VDT policies then arises from combining the elements described 
above. In each training iteration $t=1,2,\dots,T$, we sample a minibatch of source and target points, initialize the 
corresponding set of particles and perform a sequence of primal updates, and finally evaluate the gradient of the 
resulting approximation of the dual function. The method is shown in pseudocode form as 
Algorithm~\ref{alg:VDT_training}. We provide further implementation details in 
Appendix~\ref{app:implementation-details}.

\begin{algorithm}[t]
\caption{Primal-dual VDT training}\label{alg:VDT_training}
{\textbf{Input:}} Set of source and target data points $\pa{\Xsrc(i)}_{i=1}^n$, $\pa{\Xtgt(i)}_{i=1}^n$
\;
\textbf{For} $t=1,2,\dots,T$, \textbf{repeat}
    \begin{enumerate}[leftmargin=*]
     \item Sample minibatch $B = \pa{\Xsrc(i),\Xtgt(i)}_{i=1}^b$ and initialize particles via \cref{eq:particle_init}\;
     \item \textbf{Primal updates:} For $k=1,\dots,\kprimal$, update particles via \cref{eq:particle_update}\; 
     \item \textbf{Dual update:} Compute stochastic gradients via \cref{eq:dual_update} and update 
parameters as $\theta \gets \theta + \gamma \Delta_\theta$\;
    \end{enumerate}
\textbf{Return:} value function $V_\theta$\;
\end{algorithm}

\subsection{Generative modeling by value-driven transport}\label{sec:prediction}
The trained value function $V_\theta$ can be directly used for purposes of generative modeling: 
to generate samples from the target, one draws a sample point from the source $\src$ and then 
iteratively applies the learned VDT policy using \cref{eq:vdt_prediction}. 
We outline a few possible extensions to this procedure below.

\paragraph{Generation time scales.} 
\begin{wrapfigure}{r}{0.5\textwidth}
\vspace{-.4cm}
\begin{minipage}{.5\textwidth}
\begin{algorithm}[H]
\caption{VDT prediction}\label{alg:VDT_prediction}
\textbf{Input:} value function $V_\theta$, test horizon $\Htest$.\;
\textbf{Init:} draw source sample $X_0 \sim \src$.\;
\textbf{For} $h=0,1,\dots,\Htest$, \textbf{repeat}
$X_{h+1} = X_h - \frac{1}{\Htest + 1} \nabla_x V_\theta(X_h,h/\Htest)$\;
\end{algorithm}
\end{minipage}
\vspace{-.4cm}
\end{wrapfigure}
On top of the structural properties already established, we can also show that the 
optimal value function $\Vstar$ satisfies a certain invariance property with respect to the time horizon $H$ (see 
Appendix~\ref{app:few-step}). Specifically, fixing any \emph{time scale} $k >0$ and considering two different horizons 
$H$ and $H_k$ with $H_k+1 = k(H+1)$, we have that the corresponding value functions $V^{\star,H}$ and 
$V^{\star,H_k}$ are related as $V^{\star,H}_{h}(x) = V^{\star,H_k}_{h_k}(x)$ for each $h_k = kh$. This implies that a 
value function $V$ learned at a fixed $H$ can be used for generation at 
various time scales. In particular, setting $\Htest$ as the test-time horizon, one can use the policy $\pi(x) = x - 
\frac{1}{\Htest+1} \nabla_x V(x)$ for generation. To take advantage of this, it is convenient to parametrize $V_\theta$ 
so that, instead of the absolute time index $h$, it takes the normalized index $h/H$ as input. We state our
generation procedure using this convention as Algorithm~\ref{alg:VDT_prediction}.

\paragraph{Translation tasks and reverse generation.} Since our problem formulation and training method is entirely 
symmetric with respect to the roles of $\tgt$ and $\src$, the learned value function can be used for reverse 
generation: for generating samples from $\src$ given samples from $\tgt$, one simply needs to flip the 
sign and time-direction\footnote{In fact, this reverse-generation procedure can be immediately read out from 
Equation~\eqref{eq:gradient_dir}.} of the gradient updates in Equation~\eqref{eq:vdt_prediction}.
Additionally, Algorithm~\ref{alg:VDT_training} is also suitable for dealing with so-called \emph{paired}  translation 
tasks where there is a known coupling between the data sets used for training, by taking this coupling into account 
when forming the minibatches. 

\paragraph{Conditional generation.} Our framework can easily be adapted to the important setup of conditional 
generation, where each source and target point is augmented with a class label (or other supplementary information 
such as a well-chosen embedding of caption text). This additional information can be easily incorporated by passing it 
as an additional argument to the value function. The training and generation procedures are straightforward to update 
accordingly.

\section{Experiments}\label{sec:experiments}
We perform a range of numerical experiments to study the performance of our algorithm, and provide more insight about 
the properties of VDT policies. 
Instead of attempting to solve the most challenging generative modeling tasks (where competing methods have serious 
advantage due to the years of collective engineering experience they have benefited from), 
we aim to show that our method performs comparably with the best well-known methods in small-scale settings where 
thorough benchmarking of the methods themselves is possible. We also provide some early evidence that our method can 
achieve nontrivial behavior in large-scale data sets as well. Along the way, we highlight a few unique features of our 
methodology that might make it a desirable alternative to other frameworks.

\begin{figure}[t]
\centering
\setlength{\tabcolsep}{5pt}
\renewcommand\arraystretch{0.8}
\vspace{-0.3cm}
\scalebox{0.63}{
    \begin{tabular}{ccccc}
    \toprule 
     & \multicolumn{4}{c}{\textit{Wasserstein-2 distance from target}}\tabularnewline
    \cmidrule{2-5} \cmidrule{3-5} \cmidrule{4-5} \cmidrule{5-5} 
    \textit{Dataset} & moons & scurve & 8gaussians & moons-8gaussians\tabularnewline
    \midrule[1pt]
100 step VDT+ & \textbf{0.131\small{\textpm  0.034}} &  \textbf{0.120\small{\textpm  0.013}}& 0.435\small{\textpm  
0.123} & \textbf{0.652 \small{\textpm  0.151}} \tabularnewline
10 step VDT+ &  \textbf{0.132 \small{\textpm  0.029}} &  \textbf{0.125\small{\textpm  0.013}}& 0.424\small{\textpm  
0.111} 
& \textbf{0.626 \small{\textpm  0.097 }} \tabularnewline
1 step VDT+ & 0.229 \small{\textpm  0.008} &  0.262\small{\textpm  0.004}& 0.809\small{\textpm  0.025} 
& 1.365\small{\textpm  0.039}\tabularnewline
\midrule
100 step VDT & 0.219\small{\textpm  0.086} &  0.208\small{\textpm  0.016}& 0.547\small{\textpm  0.120} 
& 1.205 \small{\textpm  0.125} \tabularnewline
10 step VDT & 0.307\small{\textpm  0.035} &  0.273\small{\textpm  0.058}& 0.565\small{\textpm  0.073 } 
& 1.360 \small{\textpm  0.077 } \tabularnewline
1 step VDT & 1.497 \small{\textpm  0.046} &  1.803\small{\textpm  0.048}& 4.152 \small{\textpm  0.099} 
& 6.498 \small{\textpm  0.058}\tabularnewline
    \midrule
    SF$^2$M+ & \textbf{0.124\small{\textpm  0.023}} & \textbf{0.128\small{\textpm  0.005}} & 
\textbf{0.275\small{\textpm  0.058}} & 
{0.726\small{\textpm  0.137}}\tabularnewline
    SF$^2$M & {0.185\small{\textpm  0.028}} & {0.201\small{\textpm  0.062}} & {0.393\small{\textpm  0.054}} & 
{1.482\small{\textpm  0.151}}\tabularnewline
    DSBM-IMF++ & \textbf{0.123\small{\textpm  0.014}} & \textbf{0.130\small{\textpm  0.025}} & 
\textbf{0.276\small{\textpm  0.030}} & 
0.802\small{\textpm  0.172} \tabularnewline
    DSBM-IMF & 0.144\small{\textpm  0.024 }& 0.145\small{\textpm  0.037 }& 0.338\small{\textpm  0.091 }& 
0.838\small{\textpm  0.098}\tabularnewline
    OT-CFM+ & \textbf{0.130\small{\textpm  0.016}} & {0.144\small{\textpm  0.028}} & {0.303\small{\textpm  
0.043}} & \textbf{0.601\small{\textpm  0.027}}\tabularnewline
RF & 0.283\small{\textpm  0.045}& {0.345\small{\textpm  0.079}} & {0.421\small{\textpm  0.071}} & 
1.525\small{\textpm 0.330}\tabularnewline
    \midrule
    oracle & - & - & - & 
- \tabularnewline
    \bottomrule
    \end{tabular}
}
\scalebox{0.63}{
    \begin{tabular}{ccccc}
    \toprule 
    \multicolumn{4}{c}{\textit{Path energy}}\tabularnewline
    \midrule
    moons & scurve & 8gaussians & moons-8gaussians\tabularnewline
    \midrule[1pt] 
1.238 \small{\textpm  0.059} &  \textbf{1.629\small{\textpm  0.027}}& \textbf{14.386\small{\textpm  0.135}} 
& \textbf{30.444 \small{\textpm  0.311}}\tabularnewline
1.236 \small{\textpm  0.058} &  \textbf{1.623\small{\textpm  0.026}}& \textbf{14.290\small{\textpm  0.148}}
& \textbf{29.989\small{\textpm  0.278}}\tabularnewline
1.066 \small{\textpm  0.073} &  1.358 \small{\textpm  0.026}& 10.817 \small{\textpm  0.254} 
& 22.679 \small{\textpm  0.333}\tabularnewline
    \midrule
2.416 \small{\textpm  0.096} &  3.051\small{\textpm  0.084}& 17.799\small{\textpm  0.328} 
& 73.001 \small{\textpm  2.228}\tabularnewline
2.272 \small{\textpm  0.073} &  2.955 \small{\textpm  0.07}& 18.046\small{\textpm  0.332} 
& 76.488 \small{\textpm  2.682}\tabularnewline
1.154 \small{\textpm  0.061} &  0.82 \small{\textpm  0.036}& 0.419 \small{\textpm  0.072} 
& 43.230 \small{\textpm  4.396}\tabularnewline
    \midrule 
    \textbf{1.183\small{\textpm  0.043}} & \textbf{1.686\small{\textpm  0.039}} & {14.66\small{\textpm  0.173}} & 
\textbf{31.36\small{\textpm 0.930}}\tabularnewline
{2.08 \small{\textpm  0.146}} & {3.01\small{\textpm  0.173}} & {16.74\small{\textpm  0.274}} & {107.3\small{\textpm  
9.695}}\tabularnewline
{1.594\small{\textpm  0.043}} & 2.116\small{\textpm  0.018} & {14.88\small{\textpm  0.252}} & {41.09\small{\textpm  
1.206} }
\tabularnewline
    {1.580\small{\textpm  0.036}} & {2.092\small{\textpm  0.053}} & \textbf{14.81\small{\textpm 0.255}} & 
{41.00\small{\textpm 1.495}}\tabularnewline
    \textbf{1.216\small{\textpm  0.01}} & \textbf{1.675\small{\textpm  0.019}} & {14.88\small{\textpm  0.389}} & 
\textbf{30.47\small{\textpm 0.300}}\tabularnewline 
{1.269\small{\textpm  0.088}} & {1.793\small{\textpm  0.107}} & {15.06\small{\textpm  0.447}} & 
{36.11\small{\textpm 2.701}}\tabularnewline 
\midrule
{1.123\small{\textpm  0.01}} & {1.631\small{\textpm  0.03}} & {14.43\small{\textpm  0.045}} & {30.02\small{\textpm  
0.076}}\tabularnewline
    \bottomrule
    \end{tabular}
}
\captionof{table}{Sampling quality as measured by Wasserstein-2 distance to target and path energy for the 2D 
experiments, mean \textpm 1 standard deviation over 5 repetitions. More details and results in 
Appendix~\ref{app:more-2d-experiments}.}
\label{tab:2d_result}
\centering
\includegraphics[width=0.9\textwidth,trim=100 0 100 0]{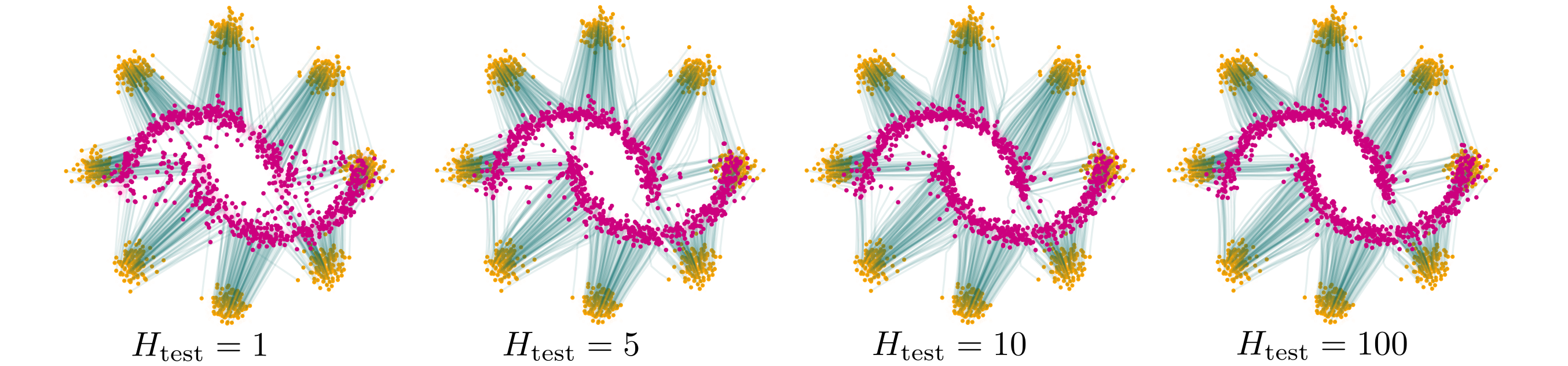}
\captionof{figure}{Few-step generation with a learned VDT model.}\label{fig:few_step}
\end{figure}

\subsection{2D experiments}
We first evaluate our method in a range of 2D benchmark experiments inspired by 
\citet{TFM+23a,TMF+23b}, also used by \citet{SdBCD23}. In our table, we reuse the figures reported by \citet{TMF+23b}. 
In most experiments, we set the source distribution as a 
standard Gaussian and vary the target distribution among three synthetic data sets (called ``moons'', ``scurve'', 
``8gauss''). In the last experiment, we set the source and target distributions as samples from the ``moons'' and 
``8gauss'' data sets. We report an estimate of the Wasserstein-2 distance\footnote{For this comparison, we omit the 
$\frac 12$ factor featured in our definitions.} between the target and the generated distributions, as well as the 
``path energy'' defined as $\sum_{h=0}^H c(X_h,X_{h+1})$ for each sample path. Since the path energy should approach 
the squared Wasserstein-2 distance for optimal transport paths, we also report estimates of this quantity in the last 
row of the table. For our own method, we report performance for two different particle-initialization schemes (na\"ive 
and OT, the latter version marked as ``VDT+''), and various choices of prediction horizon $\Htest$. We observe that 
using an OT map for particle initialization greatly improves generation quality, leading to reasonable accuracy already 
after a single generation step. Notably, the results for $\Htest = 10$ and $\Htest = H = 100$ are nearly identical, 
thus showing that one can gain a tenfold speedup in prediction time without having to make compromises about sample 
quality. 
Figure~\ref{fig:few_step} illustrates the quality of few-step generation in the same setup. Complete details and more 
results are provided in~Appendix~\ref{app:more-2d-experiments}.

\subsection{Experiments on MNIST}

We next showcase the flexibility of the VDT framework by evaluating it on three common generative modeling tasks on the 
classic MNIST data set \citep{MNIST}. A more complete description of the experimental setups and more experiments are 
reported in Appendix~\ref{app:more-mnist-experiments}.

\begin{figure}
\centering

\begin{subfigure}[b]{.45\textwidth}
\vskip 0pt 
\begin{centering}
         \includegraphics[width=\linewidth]{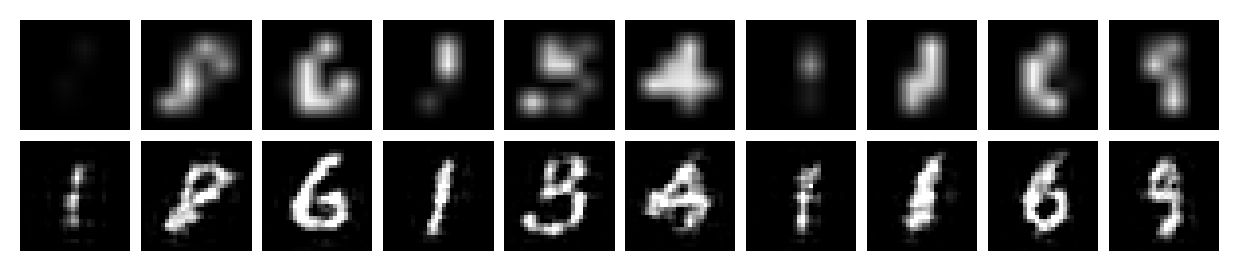}
        \includegraphics[width=\linewidth]{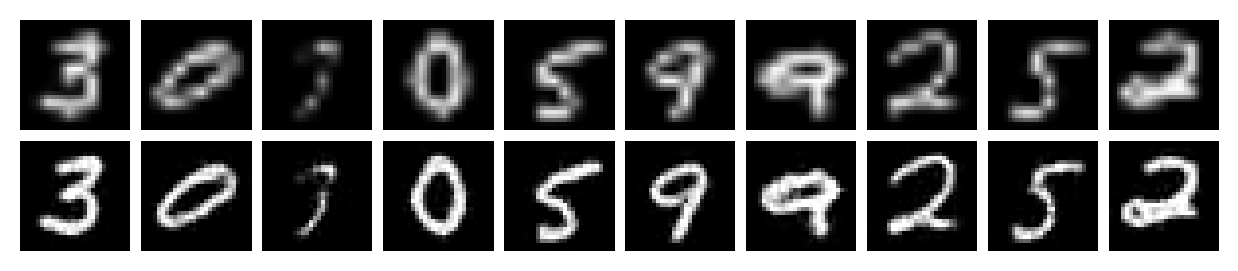}
(a) Paired data translation for deblurring downscaled digits (sizes $6\times6$ and $10\times10$).
\end{centering}

\begin{centering}
    \includegraphics[width=\linewidth]{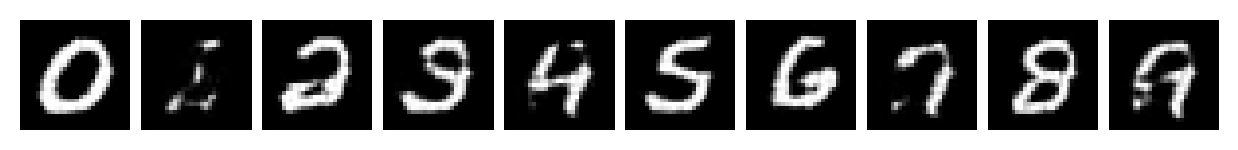}
\includegraphics[width=\linewidth]{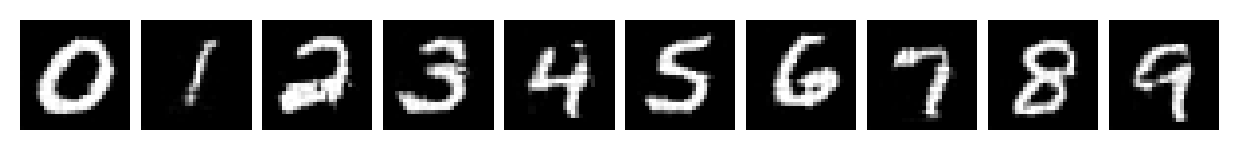}
\end{centering}
(b) Conditional generation of MNIST digits by label, with and without guidance. 
\vspace{.1cm}
\end{subfigure}
\hspace{.7cm}
\begin{subfigure}[b]{0.45\textwidth}
\vskip 0pt
\begin{centering}
\adjincludegraphics[width=\linewidth,trim={0 0 0 {.5\height}},clip]
{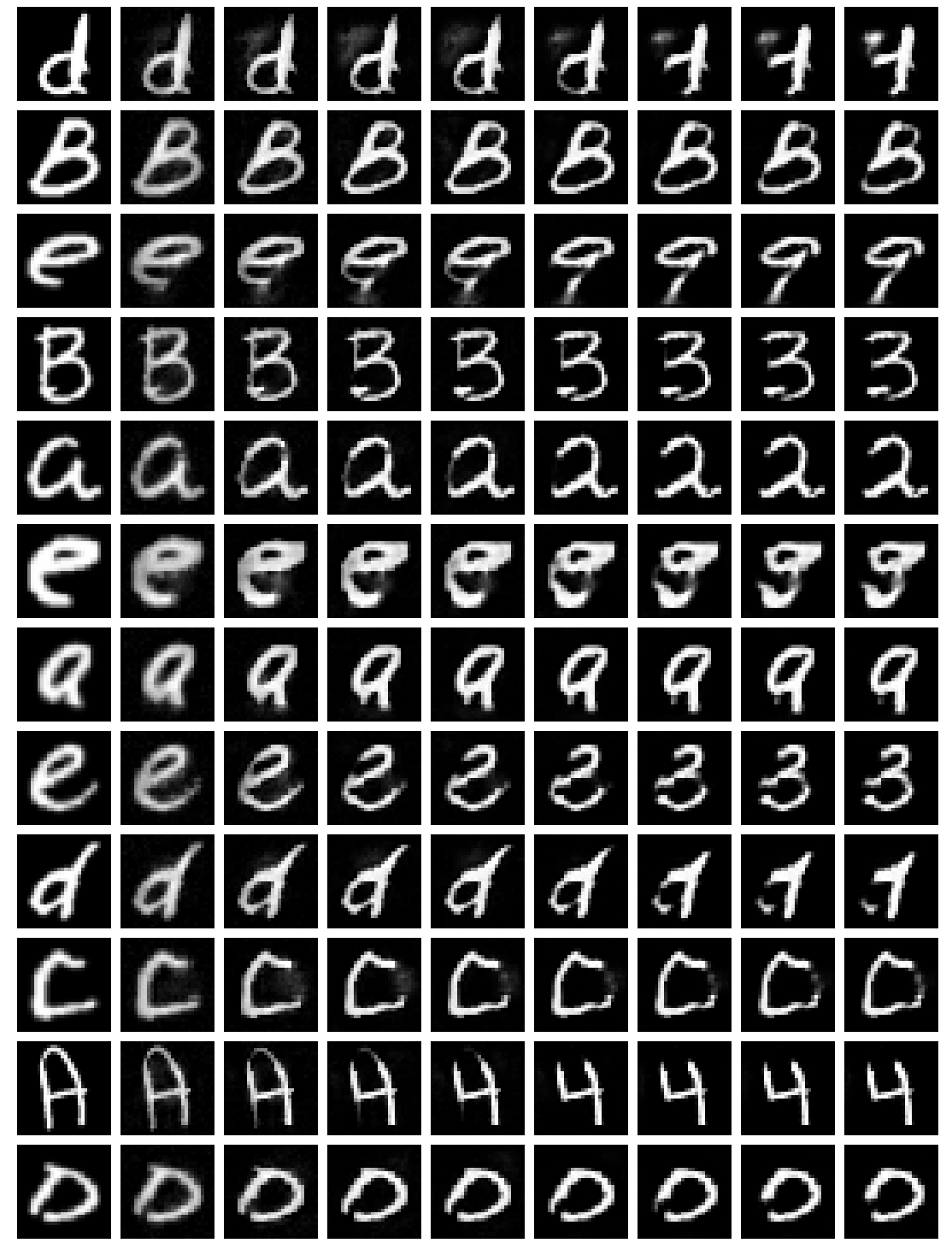}
\end{centering}
    (c) Unpaired EMNIST $\ra$ MNIST translation with few steps. Source data is shown on the left, followed by data 
obtained with different settings of $\Htest \in \ev{1,2,4,5,10,25,50,100}$.
    \label{fig:letter_digit}
\end{subfigure}
\caption{Experiments on MNIST: conditional generation and data translation.}\label{fig:mnist}
\end{figure}

\paragraph{Paired deblurring.} The first task we consider is to deblur corrupted MNIST images. We downsample 
the original $28\times28$ MNIST images to dimension $m\times m$ and linearly upsample them back to the original full 
dimension. The resulting blurred image $S(i)$ is paired with its original counterpart $T(i)$ during VDT training. We 
train a model for $m=6,10,14$ each and evaluate the result on a holdout set of unseen blurred MNIST images. As shown 
in~\cref{fig:mnist}a, the VDT policies are indeed able to reconstruct the clean image even after rather severe 
blurring.

\paragraph{Unpaired letter to digit translation.} We next consider a task without naturally available pairings, namely 
the one of translating handwritten letters (`a'-`e' and `A'-`E' from EMNIST \citep{EMNIST}) to handwritten digits (0-9 
from MNIST). We train a VDT model on unpaired minibatches of letters and digits, and show the results obtained on 
a holdout set by the few-step generation process for various choices of $\Htest$ in Figure~\ref{fig:mnist}c.
While the sharpness of the generated digits could likely be improved through heavier training, they clearly exhibit the 
desired resemblance of the source letters, demonstrating that our algorithm can learn a nontrivial relation between the 
domains without explicit supervision. Additionally, we observe a similar test-time speedup achieved by few-step 
generation with $\Htest = 10$, a factor $10$ improvement over the training horizon $H=100$.

\paragraph{Conditional image generation and guidance.}
Finally, we study the setup of conditional generation of MNIST digits. 
For this task, we augment our algorithm by concatenating the class label to the input of the parameterized value 
function and train it using the same methodology as for the other experiments. More importantly, an equally simple 
modification to the VDT training and inference procedure allows us to incorporate \emph{classifier-free guidance} (CFG) 
\citep{HS22}, originally proposed for improving conditional generation quality of diffusion models. We 
provide more details in Appendix~\ref{app:more-mnist-experiments}.
The results indicate that CFG affects the generation quality in the same way as it does for diffusion models: we obtain 
enhanced image quality at the expense of smaller sample diversity.

\section{Limitations and outlook}\label{sec:conc}
Our new framework contributes to the already very busy research area of generative modeling. In many respects, our 
approach significantly departs from the mainstream within this area---we review the most notable similarities and 
differences in Appendix~\ref{app:related-work}. Besides holding strong promises for the future, this novelty also means 
that our results come with severe limitations. The most obvious of these is that, since our primal-dual algorithm for 
computing value functions is entirely new in this context, there are no available best practices that we could adapt in 
our implementation. Lacking such foundations to build on, we have not yet been able to scale the method to meaningfully 
attack the most challenging generative modeling tasks. We are confident that further experimentation will reveal better 
ways to tune our method so that it can reach its full potential, or to identify better algorithmic ideas to solve the 
linear program at the heart of our formulation. Indeed, we strongly believe that our methodology has significant 
conceptual advantages that can be turned into practical advantages once the training methodology we propose here 
is sufficiently refined, which is a challenge that we invite the reader to address in future work.

\begin{ack}
  The authors thank the following colleagues for insightful discussions at various stages of working on this project: 
Lorenzo Croissant, Niao He, Gabriel Peyr\'e, Anabel Pichardo, and Tina Sikharulidze. Additionally, the authors would 
like to thank Javier Segovia-Aguas, Anders Jonsson, and Vicen\c{c} G\'omez for their advice and support on accessing 
the new UPF GPU cluster ``Correfoc''.
  G.~Neu was supported by the European Research Council (ERC) under the European Union’s Horizon 2020 research 
and innovation programme (Grant agreement No.~950180). Part of this research was performed while A.~M\"uller was 
visiting the Institute for Mathematical and Statistical Innovation (IMSI), which is supported by the National Science 
Foundation (Grant No. DMS-2425650). P.~Moreno-Mu\~noz was supported by ``\textit{La Caixa}'' Banking Foundation through 
the Junior Leader Postdoctoral Fellowship Programme (Grant agreement ID code.~LCF/BQ/PI24/12040025).
\end{ack}

\bibliographystyle{plainnat}
\bibliography{refs,ngbib,allbib}

\clearpage
\appendix

\part*{Appendix}
\addcontentsline{toc}{part}{Appendix} 
\etocsettocdepth{subsection} 
\localtableofcontents

\section{Related work}\label{sec:related-work}\label{app:related-work}
Our work fits into an already rich literature on formulating and solving generative modeling problems using tools of 
optimal transport. We review the most relevant previous work here, and refer the interested 
reader to one of the numerous recent surveys for a more complete treatment (some good examples being 
\citep{LHH+24,LSKME25}).

The idea of formulating generative modeling as an iterative process has been brought to the forefront of 
machine-learning research by works such as \citet{RM15}, \citet{SDWMG16} and \citet{CRBD18}.
The ensuing literature came to be dominated by continuous-time formulations, whereby generation is regarded as a 
continuous-time process where initial samples drawn from a source distribution are pushed forward along solution paths 
of ordinary differential equations (ODEs) or stochastic differential equations (SDEs). Common ideas for training such 
continuous-time models include learning time-reversals of processes that gradually transform clean data into noise 
\citep{SDWMG16,HJA20,SSDKKEP20,SME20,DN21,ND21,KAAL22,YZS+23}  and 
fitting neural networks to fixed interpolations of source and 
target data sets \citep{LCBHNL23,Liu22,LGL23,TFM+23a,AVE23,NBSM23,LHH+24,ABVE25}. These approaches make 
use of tools of continuous-time mathematics (change-of-variable formulas, Fokker--Planck equations, etc.) and 
probabilistic modeling (likelihood-based objectives, score-based denoising, etc.), which are all very distinct from the 
tools we employ in the present paper. We provide a brief comparison with the most relevant past works below.

\paragraph{Flow-based models.} Among flow-based models, a notable work we wish to highlight is the 
conditional flow matching (CFM) framework of \citet{TFM+23a}. At a high level, CFM aims at learning a transport map by 
fitting a vector field to minibatch estimates of the OT plan connecting the source and target distributions. While the 
idea of using minibatch estimates of the  OT map also appears in our algorithm, there is a major difference in how the 
two methods make use of it: while CFM tries to merely lift the (possibly heavily biased) minibatch OT map to the entire 
state space, our method only uses it as initialization for the primal updates. During training, our algorithm keeps 
moving toward the true OT map over the entire data set (up to optimization and approximation errors), without suffering 
the consequences of the possible bias due to using small minibatches. Indeed, recall that an ideal implementation of 
our method will update the value-function parameters according to unbiased gradient estimates of the dual function, 
thus converging towards its maximizer---irrespective of the minibatch size or the initialization. 

Another popular method within this family that we would like to mention is the rectified flow algorithm of 
\citet{Liu22,LGL23}. Rectified flow models aim to approximate OT maps by iteratively refining a sequence 
of flow models: first, fit a flow model to i.i.d.~samples from the true source and target distributions, then use the 
model to generate fresh, \emph{coupled} samples from the learned model, fit a new model to this, and iterate. Under 
ideal conditions, this method can be shown to converge to an OT map, the error accumulation resulting from training 
each subsequent model on the outputs of the previous model (without using any of the original samples) can be 
significant, which limits the applicability of this method for computing OT plans. A version of this method called 
$c$-rectified flow \citep{Liu22} aims at establishing a more direct link with the dynamic OT solution path 
by parametrizing the learned vector field as a gradient map of a scalar function (which is the form suggested by the 
Benamou--Brenier theorem). While this idea is attractive, it makes for a rather impractical method as it requires 
evaluating second derivatives of neural networks during training.  In contrast, our method natively takes into account 
the structure of the optimal policy and optimizes estimates of the value function (a.k.a., the potential function). 

\paragraph{Schr\"odinger bridges.} Another important class of generative modeling tools make use of the concept of 
\emph{Schr\"odinger bridges} (SB), defined as stochastic processes $(X_t)_{t\in[0,1]}$ with marginals $\Law(X_0) = 
\src$, $\Law(X_1) = \tgt$ and minimal relative entropy to a fixed reference process. Schr\"odinger bridges are closely 
linked with an entropy-regularized version of the Kantorovich OT problem (often called ``entropic OT'', EOT), in a way 
analogous to the relation between dynamic and static OT (which we discussed in some detail in 
Section~\ref{sec:prelim}). 
Taking advantage of the many beautiful mathematical results about entropic OT, several successful generative models 
have been developed for approximating Schr\"odinger bridges 
\citep{dBTHD21,SdBCD23,TMF+23b,LVHTNA23,GKBK24,GSKBK24,dBKMD24}. While, like dynamic OT solution 
paths, Schr\"odinger bridges can also be succinctly parametrized by a single scalar-valued potential function (analogous 
to our value function), we are not aware of any method that utilizes this fact as straightforwardly as our method does. 
Indeed, SB-based methods typically feature multiple neural networks for parametrizing the control drifts (often 
separately in both time directions) and the dual variables for enforcing the source and target constraints, and use a 
variety of ideas from score-based generative modeling and flow matching during training. While these methods achieve 
outstanding empirical performance for a range of generative modeling tasks, the connection between their training 
methodologies and their core problem formulation is arguably often quite loose---making this a fruitful area of 
research with several interesting open problems.

\section{Discrete-time dynamic optimal transport} \label{app:lp-properties}

In this section, we provide a complete treatment of the discrete-time dynamic OT problem that our method is based on. 
To keep this section self-contained and easy to read, we will repeat some content already covered in the main text. 

\subsection{Definitions of OT problems}

\paragraph{Definitions.} Let $\src, \tgt$ be probability measures on $\R^d$. For probability measures $\gamma$ over 
$\R^d\times\R^d$, consider the space of all (``marginal-preserving'') \emph{couplings} $$\Gamma(\src, \tgt) = \{\gamma 
\in \Sc \mid \bwd_{\#}\gamma = \src \wedge \fwd_{\#}\gamma = \tgt \}.$$ Consider the \textit{Kantorovich problem} 
\begin{align}
    \Wc^2_2(\src,\tgt) = \min_{\gamma \in \Gamma(\src,\tgt)} \frac 12 \int \norm{x-y}^2 \dd\gamma(x,y). 
\label{eq:kantorovich-reprise}
\end{align}

If $\src, \tgt$ have densities $\dSrc, \dTgt$, this problem is known to admit the unique minimizer $\gamma^\star = 
(\text{id}, M^\star)_{\#} \src$, where $M^\star$ is the unique minimizer of the \textit{Monge problem}
\begin{align}
    \min_{\substack{M \colon \R^d \to \R^d \colon  \\ M_{\#}\src = \tgt}} \frac 12 \int \norm{M(x) - x}^2 \dd\src(x) , 
\label{eq:monge-reprise}
\end{align}
and the optimal values of the programs coincide. 

In analogy to the continuous-time dynamic Benamou--Brenier formulation of OT \citep{benamou2000computational}, we 
define the \textit{discrete time dynamic OT problem} as follows. We consider a discrete-time stochastic control 
problem, where a sequence of states are generated recursively over a sequence of discrete \emph{time steps} 
$h=0,1,\dots,H$ (where $H$ is called the \emph{horizon}). 
The initial state $X_0$ is drawn from the source distribution $\src$, and each consecutive state 
$X_1,X_2,\dots,X_{H+1}$ is selected according to a \emph{control policy} $\pi$ that maps each state $X_h$ to the next 
state as $X_{h+1} = \pi_h(X_h)$. The control cost associated with each state transition is $c(X_h,X_{h+1}) = 
\frac{H+1}{2} \norm{X_h - X_{h+1}}^2$. In general, we will consider non-stationary policies represented as the 
collection of 
deterministic maps $\pi = \ev{\pi_h:\real^d \ra \real^d}_{h=0}^H$.
A policy is called \emph{feasible} if the law of the final state $X_{H+1}$ generated by the above process matches the 
target distribution ($\Law(X_{H+1}) = \tgt$), and we denote the set of feasible control policies by $\Pi(\src,\tgt)$. 
Letting $\EEpi{\cdot}$ stand for the expectation operator under the joint distribution of state sequences generated by 
$\pi$, we then set up our objective as
\begin{align}
    \Wc^2_{\text{dyn-M},2}(\src,\tgt) = \min_{\pi \in \Pi(\src,\tgt)} \frac{H+1}{2} \sum_{h=0}^{H} 
\EEpi{\norm{X_{h+1} - X_h}^2} 
\label{eq:DDOT-reprise}.
\end{align}
We reformulate this optimization problem using the concept of \emph{occupancy measures}, borrowed from the literature 
on optimal control in discrete-time Markov decision processes (see, e.g., Chapter 6.9 in \citet{Puterman1994}).
The occupancy measure associated with a control policy $\pi$ is the collection of joint distributions $\mupi = 
\ev{\mupi_h \in \Delta_{\real^d \times \real^d}}_{h=0}^{H}$ over the state space, with $\mupi_h = \PPpi{(X_h,X_{h+1}) 
\in \cdot}$ corresponding to the distribution of the state pair $X_h,X_{h+1}$ generated by policy $\pi$ under the 
discrete-time control process described above. As 
is well-known in the context of Markov decision processes, the set of all valid occupancy measures that can be induced 
by control policies is uniquely characterized by a set of linear constraints. 
Adapting these results to our needs, we formulate our optimal control problem as the following 
linear program (LP):
\begin{equation}
\begin{split}
    \Wc^2_{\text{dyn-K},2}(\src,\tgt) = \min_{\mu \in \Sc^{H+1}} \frac{H+1}{2}&\sum_{h=0}^{H} \int 
\|x-y\|^2 
\dd\mu_h(x,y) \\ 
    \text{s.t.} \quad\quad \src &= \bwd_{\#} \mu_0\\
    \qquad\;\fwd_{\#} \mu_h &= \bwd_{\#} \mu_{h+1} \qquad (\forall h \in \{0,\dots,H-1\})\\
    \qquad\,\fwd_{\#} \mu_H &= \tgt 
    \label{eq:primal-lp-reprise}
\end{split}
\tag{Primal LP}
\end{equation}
To distinguish between the two formulations, we will refer to the first as the \emph{Monge-type dynamic OT} problem, 
and the second one as the \emph{Kantorovich-type dynamic OT problem}, which naming convention is justified by noticing 
that the first problem reduces to the Monge-type OT problem~\eqref{eq:monge-reprise} and the second to the 
Kantorovich-type problem~\eqref{eq:kantorovich-reprise}. Throughout the appendix, we will refer to this LP as the 
\emph{primal LP}, which otherwise coincides with \eqref{eq:primal-lp} shown in the main text.

\subsection{Basic properties of the OT problems} We first establish that the Kantorovich 
problem~\eqref{eq:kantorovich-reprise} is equivalent to the Kantorovich-type dynamic OT problem 
in~\eqref{eq:primal-lp-reprise}:
\begin{theorem} \label{thm:static-dyn-kont}
    We have 
    \begin{align*}
        \Wc_{\text{dyn-K},2}^2(\src,\tgt) = \Wc_2^2(\src,\tgt).
    \end{align*}
    Moreover, if $\gamma^\star$ is optimal for \cref{eq:kantorovich}, let $(X_0,X_{H+1})\sim \gamma^\star$ and $X_h = 
X_0 + \frac{h}{H+1} (X_{H+1}-X_0)$. Then, the laws $\mu_h^\star$ of $(X_h,X_{h+1})$ form an optimal solution 
$\mu^\star$ for~\eqref{eq:primal-lp-reprise}. 
\end{theorem}

\begin{proof}
    We start by showing $\Wc_{\text{dyn-K},2}^2(\src,\tgt) = \Wc_2^2(\src,\tgt)$. To this end, let $\mu \in \Sc^{H+1}$ 
be feasible for~\eqref{eq:primal-lp-reprise}. By inductively applying the gluing lemma (e.g., \citep[Lemma 
5.5]{thorpe2018introduction}), we can pick a measure $\bar{\mu}$ over sequences in $\X^{H+2}$ and random variables $Y=(Y_0, 
\dots, Y_{H+1})\sim \bar{\mu}$ such that for all $h\in\{0,\dots,H\}$, $(Y_h,Y_{h+1})\sim \mu_h$.
    Now the dynamic objective for this choice is 
    \begin{align*}
        \frac{H+1}{2}\sum_{h=0}^H \int \|x-y\|^2 \dd\mu_h(x,y) =& \frac{(H+1)^2}{2}\E\bs{\frac{1}{H+1}\sum_{h=0}^H 
\norm{Y_h-Y_{h+1}}^2}\\
        \geq& \frac{(H+1)^2}{2}\E\bs{\norm{ \frac{1}{H+1}\sum_{h=0}^H (Y_h-Y_{h+1}) }^2} \tag{by Jensen's}\\
        =& \frac 12 \E\bs{\norm{ Y_0 -Y_{H+1}}^2}\\
        =& \frac 12 \int \norm{x-y}^2 d\gamma(x,y),
    \end{align*}
    where in the last step $\gamma$ is the law of $(Y_0,Y_{H+1})$ (formally obtained by marginalization of $\bar{\mu}$). 
It satisfies $\bwd_{\#} \gamma = \src$ and $\fwd_{\#} \gamma = \tgt$, so the above implies 
    \begin{align*}
         \frac{H+1}{2}\sum_{h=0}^H \int \|x-y\|^2 \dd\mu_h(x,y) \geq \Wc_2^2(\src,\tgt).
    \end{align*}
    Thus, we obtain that $\Wc_{\text{dyn-K},2}^2(\src,\tgt) \geq \Wc_2^2(\src,\tgt)$ after taking the infimum over all 
feasible $\mu$ on the left-hand side.

    Next, we show $\Wc_{\text{dyn-K},2}^2(\src,\tgt) \leq \Wc_2^2(\src,\tgt)$ and the optimality of $\mu^\star$. 
Let $(X_0,X_{H+1})\sim \gamma^\star$ and take $X_h = X_0 + \frac{h}{H+1}(X_{H+1}-X_0)$ as in the theorem statement, 
with the associated marginal distributions denoted by $\mu^\star_h=\text{law}(X_h,X_{h+1})$ . It is immediate that 
$\mu^\star$ respects all of the constraints in~\eqref{eq:primal-lp-reprise}. Moreover, we have
    \begin{align*}
        \frac{H+1}{2}\sum_{h=0}^H \int \|x-y\|^2 \dd\mu^\star_h(x,y) =& \E\bs{\frac{H+1}{2}\sum_{h=0}^H 
\norm{X_h-X_{h+1}}^2}\\
        =& \E\bs{\frac{H+1}{2}\sum_{h=0}^H \norm{\frac{1}{H+1}(X_{0}-X_{H+1})}^2}\\
        =& \frac{1}{2}\E\bs{\norm{(X_{0}-X_{H+1})}^2}\\
        =& \frac 12 \int \norm{x-y}^2 d\gamma^\star(x,y)\\
        =& \Wc_2^2(\src,\tgt).
    \end{align*} 
    This proves the claimed inequality. Together with the reverse inequality we have proved first, we can also see 
that in fact $\Wc_{\text{dyn-K},2}^2(\src,\tgt)$ is equal to the value obtained by $\mustar$, which implies that it is 
an optimal solution for~\eqref{eq:primal-lp-reprise}.
\end{proof}

We remark that this proof is very similar to the original proof of the Benamou-Brenier formula but without the need of 
change of measure arguments, with the discrete-time flow constraints replacing the continuity equation. We discuss 
this relationship in more detail later in this section.

We next establish the connection to the Monge-type dynamic OT problem~\eqref{eq:DDOT-reprise}. More concretely, we 
can now show the existence of deterministic optimal policies under suitable conditions by leveraging the following 
standard result from static OT theory.
\begin{lemma}[\citet{San15}, Lemma 4.23] \label{lem:no-cross}
    Let $M^\star$ be an optimal solution to the static Monge problem in~\eqref{eq:monge} and $\gamma=(\text{id}, 
M^\star)_{\#}\src$. Let $t\in(0,1)$ and $y\in\R^d$. Then, there is at most one pair $(x,z)\in\supp(\gamma)$ such that 
$y=(1-t)x + t z$. 
\end{lemma}
This fact (originally due to \citet{McC97}) implies that the optimal interpolants of the static Monge problem 
do not cross, which leads to the following result.
\begin{theorem}\label{thm:det-policies}
    Suppose that the static Monge problem~\eqref{eq:monge} admits a solution $M^\star$. Then, the dynamic Monge-type 
problem in~\eqref{eq:DDOT-reprise} admits a optimal solution $\pi^\star$.
    Moreover, if $M^\star$ is unique on $\supp(\src)$, then each $\pi^\star_h$ is uniquely defined on $\supp 
((1-h/(H+1))\text{id} + h/(H+1) M^\star)_{\#}\src$ with an expression given by
    \begin{align*}
        \pi^\star_h(x_h) =  s_h(x_h) + \frac{h+1}{H+1}\pa{M^\star(s_h(x_h))-s_h(x_h)},
    \end{align*}
    where $s_h(x_h) = \bs{(1-\frac{h}{H+1})\text{id} + \frac{h}{H+1} M^\star}^{-1}(x_h)$.
\end{theorem}

\begin{proof} We prove the two claims regarding existence and uniqueness of $\pistar$ below.

    \emph{Existence.} Consider $\pi^\star$ as in the theorem statement and let $x_0\in\supp(\src)$. Denote the 
interpolation map by $f_h(x)=(1-h/(H+1))\text{id} + h/(H+1)M^\star$. Note that $s_0(x_0)=x_0$ and thus 
$x_1=\pi^\star_0(x_0) = x_0 + 1/(H+1)(M^\star(x_0)-x_0) = f_1(x_0) \in \supp((f_1)_{\#}\src)$. Now, let us consider a 
fixed $h$ and let $x_h = f_h(x_0)$. Then by \cref{lem:no-cross}, $s_h$ is well-defined and $s_h(x_h) = x_0$. Hence, 
$\pi_h^\star(x_h) = x_0 + (h+1)/(H+1)(M^\star(x_0)-x_0) = f_{h+1}(x_0)$. This shows inductively that for 
$x_0\in\supp(\src)$, the sequence given via $x_{h+1}=\pi^\star_h(x_h)$ is well-defined and in fact satisfies 
$x_h=f_h(x_0)$ for all $h$. It follows that $\pi^\star$ is feasible for~\eqref{eq:DDOT-reprise} since 
$\pi^\star_H\circ\cdots\circ \pi^\star_0 = M^\star$ on $\supp(\src)$. Moreover, the objective is 
\begin{align*}
        \EEs{\frac{H+1}{2} \sum_{h=0}^H \norm{X_{h+1}-X_h}^2}{\pistar} 
=& \E\bs{\frac{H+1}{2}\sum_{h=0}^H \norm{\frac{1}{H+1} (M^\star(X_0)-X_0)}^2}\\
        =& \frac 12 \E\bs{\norm{M^\star(X_0)-X_0}^2}\\
        =& \Wc_2^2(\src,\tgt).
    \end{align*}
    where we used that $M^\star$ realizes $\Wc_2^2(\src,\tgt)$. On the other hand, introducing the shorthand 
notation $\pibar_h = \pi_h \circ \dots \circ \pi_0$, we can see that any feasible policy $\pi$ must satisfy  
    \begin{align*}
        &\EEpi{\frac{H+1}{2}\sum_{h=0}^H \norm{X_{h+1}-X_h}^2} \\
        &\qquad\qquad= \frac{(H+1)^2}{2}\EE{\frac{1}{H+1}\sum_{h=0}^H \norm{\pibar_h(X_0)-\pibar_{h-1}(x_0)}^2}\\
        &\qquad\qquad\geq 
\frac{(H+1)^2}{2}\EE{\norm{\frac{1}{H+1}\sum_{h=0}^H(\pibar_h(X_0)-\pibar_{h-1}(X_0))}^2} \tag{by 
Jensen's inequality}\\
        &\qquad\qquad= \frac 12 \EE{\norm{X_0 - \pi_H\circ\cdots \pi_0 (X_0)}^2}\\
        &\qquad\qquad\ge \Wc_2^2(\src,\tgt),
    \end{align*}
    where in the final step we used that $\pi_H\circ\cdots\circ \pi_0$ is feasible for~\eqref{eq:monge} since $\pi$ is 
feasible for~\eqref{eq:primal-lp-reprise} in the final step. This implies that $\pi^\star$ is indeed optimal 
for~\eqref{eq:primal-lp-reprise}.

    \emph{Uniqueness.} If $\pi$ is optimal, then (since Jensen's holds point-wise within the expectation), the above 
chain of inequalities shows that for any realization of $x_0$ and the induced sequence $x_{h+1}=\pi_h(x_h)$, we must 
have $x_{h+1}-x_h = x_1-x_0$ for any $h$. Thus we have 
    \begin{align*}
        \EEpi{(H+1)\sum_{h=0}^H \norm{X_{h+1}-X_h}^2}
        =& \EE{\norm{X_0 - \pi_H\circ\cdots \pi_0 (X_0)}^2}.
    \end{align*}
    Since $\pi$ is feasible, this equals $\Wc_2^2(\src,\tgt)$ if and only if $\pi_H\circ\cdots \circ \pi_0$ is optimal 
for the static problem~\eqref{eq:monge}, and thus must satisfy $\pi_H\circ\cdots \pi_0=M^\star$ on $\supp(\src)$. Since 
$x_{h+1}-x_h=x_1-x_0$, it follows that $\pi_h(x_h)=\pi^\star_h(x_h)$ on $\supp((f_h)_{\#}\src)$.
\end{proof}

Putting Theorems~\ref{thm:static-dyn-kont} and~\ref{thm:det-policies} together immediately implies that the Monge- and 
Kantorovich-type dynamic OT problems are equivalent. We state this formally and connect the 
solutions $\mustar$ and $\pistar$ to each other explicitly below.

\begin{corollary} \label{cor:monge-to-kant}
    Suppose that the static Monge problem~\eqref{eq:monge} admits the unique (on $\supp(\src)$) solution $M^\star$. 
Let $\pi^\star$ be the optimal solution of the dynamic Monge-type problem~\eqref{eq:DDOT-reprise}. Then, $\mu^\star$ 
given by $\mu^\star_h = (\pi^\star_{h-1}\circ\cdots\circ \pi^\star_0 ,\, \pi^\star_{h}\circ\cdots\circ 
\pi^\star_0)_{\#} \src$ (with the convention $\pi^\star_{-1}\circ\cdots\circ \pi^\star_0=\text{id}$) is an optimal 
solution for the dynamic Kantorovich-type problem~\eqref{eq:primal-lp-reprise} and has the same value. 
\end{corollary}

\begin{proof}
Let $\mu^\star$ be defined as in the statement. By construction, we 
have $\bwd_{\#}\mu^\star_0=\text{id}_{\#}\src=\src$ and $\fwd_{\#}\mu^\star_h=(\pi^\star_h\circ\cdots\circ 
\pi^\star_0)_{\#}\src=\bwd_{\#}\mu^\star_{h+1}$, as well as $\fwd_{\#}\mu^\star_H = (\pi^\star_H\circ\cdots\circ 
\pi^\star_0)_{\#}\src=\tgt$. Thus,  $\mu^\star$ is feasible for~\eqref{eq:primal-lp-reprise}. 
    
Let $W(\mu^\star)$ be the value of~\eqref{eq:primal-lp-reprise} for this $\mu^\star$. Now, let us consider the set of 
random variables defined recursively as $X_0\sim\src$ and $X_{h+1}=\pi^\star_h(X_{h})$ for each $h$. Then, the pair 
$(X_h,X_{h+1})$ is distributed as $\mu^\star_h$ and hence the value $W(\mu^\star)$ coincides 
with that of~\eqref{eq:primal-lp-reprise} for the optimal $\pi^\star$. Hence, as seen in the proof of 
\cref{thm:det-policies}, $M^\star=\pi^\star_H\circ\cdots\circ \pi^\star_0$ is optimal for the static Monge problem in 
\cref{eq:monge} with value $W(\mu^\star)$. The latter thus also equals the value of the static Kantorovich problem 
in \cref{eq:kantorovich}. In turn, \cref{thm:static-dyn-kont} thus shows that the optimal value of the Kantorovich-style 
dynamic OT problem of~\eqref{eq:primal-lp-reprise} is indeed $W(\mu^\star)$. Hence $\mu^\star$ is optimal 
for~\eqref{eq:primal-lp-reprise}. 
\end{proof}

Altogether, the results in this section prove the claims made in Theorem~\ref{thm:DOT-structure}.

\subsection{Value functions and the Bellman equations} 

In the following, we restrict the problem to the domains $\Mc=\Mc_{\geq 0}(\Omega\times\Omega)$ and $\Cc(\Omega)$ over 
a compact convex set $\Omega\subset\R^d$ rather than all of $\R^d$. That is, we assume that $\supp(\src), \supp(\tgt) \subset 
\Omega$. Then, all previous statements then hold analogously for the optimization problems over these domains, with the addition that the maps $\pi^\star$ in 
\cref{thm:det-policies} are continuous (see \citep[Lemma 5.29]{santambrogio20151}).\footnote{This is needed to have 
measures and continuous functions to be properly paired spaces, and to establish duality.}

We define the \emph{Lagrangian} associated with the constrained optimization problem~\eqref{eq:primal-lp-reprise} as 
\begin{align*}
    \Lc(\mu, V) =& \frac{H+1}{2}\sum_{h=0}^H \int \norm{x-y}^2 \dd\mu_h(x,y) + \int V_0(x) 
\dd(\src-\bwd_{\#}\mu_0)(x)\\ 
    &+ \sum_{h=1}^H \int V_h(x) \dd(\fwd_{\#}\mu_{h-1}-\bwd_{\#}\mu_h)(x) + \int V_{H+1}(x) \dd(\fwd_{\#}\mu_H - 
\tgt)(x)\\
    =& \sum_{h=0}^H \int \br{\frac{H+1}{2}\norm{x-y}^2 + V_{h+1}(y)-V_h(x)} \dd\mu_h(x,y)\\
    &+ \int V_0(x) \dd\src(x) - \int V_{H+1}(x)\dd\tgt(x)
\end{align*}
and consider the (primal) Lagrangian formulation equivalent to the dynamic OT problem:
\begin{align}
    \Wc^2_{\text{dyn},2}(\src,\tgt) =  \inf_{\mu\in\Mc^{H+1}} \sup_{V\in\Cc(\Omega)^{H+2}}~ \Lc(\mu,V). 
\tag{P}\label{eq:primal-lagr}
\end{align}
Note that we relaxed the condition on $\mu$ to unnormalized measures $\mu_h\in\Mc=\Mc_{\geq 
0}(\Omega\times\Omega)$ since the feasibility set of the dynamic Kantorovich problem already ensures that $\mu$ is 
normalized. Indeed, notice that by the first marginal constraint, we have 
$1=\src(\Omega)=[\bwd_{\#}\mu_0](\Omega)=\mu_0(\bwd^{-1}(\Omega))= \mu_0(\Omega\times\Omega)$. Inductively, if 
$1=\mu_h(\Omega\times\Omega)$, using the flow constraint we have 
$1=\mu_h(\Omega\times\Omega)=\mu_h(\fwd^{-1}(\Omega))=[\fwd_{\#}\mu_h](\Omega)=[\bwd_{\#}\mu_{h+1}](\Omega)=\mu_{h+1}
(\bwd^{-1}(\Omega))=\mu_{h+1}(\Omega\times\Omega)$.

We can now consider the \emph{dual} of this problem, defined as
\begin{align}
    \sup_{V\in\Cc(\Omega)^{H+2}} \inf_{\mu\in\Mc^{H+1}}~ \Lc(\mu,V), \tag{D}\label{eq:dual-lagr}
\end{align}
and show that it can be equivalently formulated as the \emph{dual linear program} stated below.

\begin{lemma}\label{eq:dyn-dual}
    The dual problem in \cref{eq:dual-lagr} is equivalent to
    \begin{equation}
    \tag{Dual LP}
    \begin{split}
        \max_{V\in\Cc(\Omega)^{H+2}}& \int V_0(x)d\src(x)-\int V_{H+1}(x)\dd\tgt(x) \\
        \text{s.t.} \qquad& V_h(x) \leq \frac{H+1}{2}\norm{x-y}^2 + V_{h+1}(y) \qquad (\forall h \in {0,\dots,H}~ 
\forall x,y\in\Omega).
\label{eq:dual-lp}
    \end{split}
    \end{equation}
\end{lemma}
This can be seen to be analogous to the classic dual LP for optimal control in Markov decision processes, with the 
difference that 
the objective involves the target measure as a new element (rather than merely the value $V_0$ averaged over an initial 
state distribution). The derivation of the dual LP is a special case of general conic LPs \citep{shapiro2001duality}.

\begin{proof}[Proof of \cref{eq:dyn-dual}]
    Writing the Lagrangian in its ``adjoint'' form, we have 
    \begin{align*}
        \Lc(\mu, V) =& \sum_{h=0}^H \int \br{\frac{H+1}{2}\norm{x-y}^2 + V_{h+1}(y)-V_h(x)} \dd\mu_h(x,y)\\
        &+ \int V_0(x) \dd\src(x) - \int V_{H+1}(x)\dd\tgt(x).
    \end{align*}
    Clearly, we have $\inf_{\mu\in \Mc^{H+1}} \Lc(\mu,V) = -\infty$ whenever there is $h$ and a set 
$A\subset\Omega\times\Omega$ of strictly positive Lebesgue measure such that $V_h(x) > \frac{H+1}{2}\norm{x-y}^2 + 
V_{h+1}(y)$ for all $(x,y)\in A$. On the other hand, if this is not the case, then $\Lc(\mu,V) \geq \int V_0(x) 
\dd\src(x) - \int V_{H+1}(x)\dd\tgt(x)$ for all $\mu\in\Mc^{H+1}$, with equality holding for instance when all $\mu_h = 
0$. Thus $\inf_{\mu\in\Mc^{H+1}} \Lc(\mu,V) = \int V_0(x) \dd\src(x) - \int V_{H+1}(x)\dd\tgt(x)$. Hence $V$ is optimal 
for the dual problem~\eqref{eq:dual-lagr} if and only if it solves the LP in the statement.
\end{proof}

Recall (e.g., from \citep{peyre2025optimal}) the \emph{dual of the static OT problem} in \cref{eq:kantorovich} is given by 
\begin{equation}
 \begin{split}
    \max_{W_0,W_1\in\Cc(\Omega)}& \int W_0(x)\dd\src(x) - \int W_1(x)\dd\tgt(x) \label{eq:kantorovich-dual}\\
    \text{s.t.} \qquad& W_0(x) \leq \frac 12 \norm{x-y}^2 + W_1(y) \qquad (\forall x,y\in\R^d).
 \end{split}
\end{equation}
This is sometimes equivalently written as by directly setting $W_1=-(W_0)^\star$ to be the negative 
$\norm{\cdot}^2$-transform of $-W_0$ and optimizing only over $W_0$. We now show that the dual of the dynamic OT problem 
is equivalent to the dual of the static problem, and then use this to establish strong duality.

\begin{theorem}[Equivalence to the Static Dual] \label{thm:dual-stat-dyn}
    Let $V^\star$ be optimal for the dynamic dual problem~\eqref{eq:dual-lp}. Then, $W_0 = V_0^\star$ and $W_1 = 
V_{H+1}^\star$ form an optimal solution for the static dual problem~\eqref{eq:kantorovich-dual} and in particular the 
optimal values coincide.
\end{theorem}

\begin{proof}
    \emph{Feasibility:} For any $x,y\in\R^d$, set $x_h = (1-h/(H+1))x + h/(H+1)y$ (for $h=0,\dots,H+1$) and notice that
    \begin{align*}
        W_0(x) = V_0^\star(x_0) \leq& \frac{H+1}{2}\norm{x_0-x_1}^2 + V_1^\star(x_1)\\
        \leq& \dots \\
        \leq& \frac{H+1}{2}\sum_{h=0}^H \norm{x_h-x_{h+1}}^2 + V_{H+1}^\star(x_{H+1})\\
        =& \frac{H+1}{2}\sum_{h=0}^H \norm{\frac{1}{(H+1)}(x-y)}^2 + V_{H+1}^\star(x_{H+1})\\
        =& \frac 12 \norm{x-y}^2 + W_1(y),
    \end{align*}
    so $(W_0,W_1)$ is feasible for the problem in \cref{eq:kantorovich-dual}.

    \emph{Optimality:} Assume by contradiction that $W^\star$ is optimal for the problem in \cref{eq:kantorovich-dual} 
and satisfies
    \begin{align}
        \int W^\star_0(x)\dd\src(x) - \int W^\star_1(x)\dd\tgt(x) >& \int W_0(x)\dd\src(x) - \int W_1(x)\dd\tgt(x) 
\nonumber\\
        =& \int V^\star_0(x)\dd\src(x) - \int V^\star_{H+1}(x)\dd\tgt(x). \label{eq:contradtict1}
    \end{align}
    Let us proceed by defining the sequence $V'_{H+1}= W^\star_1$ and $V'_h(x) = \min_{y\in\Omega} 
\br{\frac{H+1}{2}\norm{x-y}^2 + V'_{h+1}(y)}$ inductively for $h=H,\dots,0$. Clearly, $V'$ is feasible for the dynamic 
dual problem~\eqref{eq:dual-lp}. Moreover, by construction, for any $x_0\in\R^d$ we have
    \begin{align*}
        V'_0(x_0) =& \min_{x_1}\br{\frac{H+1}{2}\norm{x_0-x_1}^2 + V'_1(x_1)}\\
        =& \dots\\
        =& \min_{x_1,\dots,x_{H+1}}\br{\frac{H+1}{2}\sum_{h=0}^H \norm{x_h-x_{h+1}}^2 + V'_{H+1}(x_{H+1})}\\
        \geq& \min_{x_1,\dots,x_{H+1}}\br{\frac{(H+1)^2}{2} \norm{\frac{1}{H+1}\sum_{h=0}^H(x_h-x_{h+1})}^2 + 
V'_{H+1}(x_{H+1})} \tag{by Jensen's}\\
        =& \min_{x_{H+1}}\br{\frac 12 \norm{x_0-x_{H+1}}^2 + W_1^\star(x_{H+1})}.
    \end{align*}
    Now, when $x_0\in\supp(\src)$, we know that the right-hand side equals $W_0^\star(x_0)$ \citep{peyre2025optimal}. 
Hence, we have $V'_0\geq W_0$ (on $\supp(\src)$) and $V'_{H+1}=W_1$. Thus, since $V'$ is feasible for the dynamic dual 
and $V^\star$ is optimal, we have
    \begin{align*}
        \int V^\star_0(x)\dd\src(x) - \int V^\star_{H+1}(x)\dd\tgt(x) \geq& \int V'_0(x)\dd\src(x) - \int 
V'_{H+1}(x)\dd\tgt(x) \\
        \geq& \int W^\star_0(x)\dd\src(x) - \int W^\star_{1}(x)\dd\tgt(x),
    \end{align*}
    contradicting the inequality~\eqref{eq:contradtict1}, thus proving our original claim.
\end{proof}

We  thus obtain that the dynamic OT problem inherits strong duality from the static OT problem.

\begin{lemma}[Strong duality]\label{lem:sd}
    The values of~\eqref{eq:primal-lp-reprise} and~\eqref{eq:dual-lp} coincide.
\end{lemma}

\begin{proof}
    Let $\opt_{\text{P}}$ and $\opt_{\text{D}}$ denote the optimal values of the static primal (\cref{eq:kantorovich}) 
and dual problems \cref{eq:kantorovich-dual}, respectively. By strong duality of the static problem 
\citep{peyre2025optimal}, we have $\opt_{\text{P}} = \opt_{\text{D}}$. Further, let $\opt_{\text{P}}^{\text{dyn}}$ and 
$\opt_{\text{D}}^{\text{dyn}}$ denote the optimal values of the dynamic primal (\cref{eq:primal-lagr}) and dual 
(\cref{eq:dual-lagr}), respectively. By \cref{thm:static-dyn-kont}, we have $\opt_{\text{P}}^{\text{dyn}} = 
\opt_{\text{P}}$. By \cref{thm:dual-stat-dyn}, we have $\opt_{\text{D}}^{\text{dyn}} = \opt_{\text{D}}$. Hence 
$\opt_{\text{P}}^{\text{dyn}} = \opt_{\text{D}}^{\text{dyn}}$.
\end{proof}

By strong duality, we can now show that the dual program admits a dynamic programming interpretation, connecting it to 
the policies of the Monge-type dynamic OT problem.

\begin{theorem} \label{thm:bellman-reprise} [Theorem~\ref{thm:bellman} in the main text]
    Let $\mu^\star$ be an optimal solution for (\ref{eq:primal-lp-reprise}) that corresponds to a deterministic policy 
$\pi^\star$. Let $V^\star$ be an optimal solution for (\ref{eq:dual-lp}). Then for all $h$ and for 
$(\pi^\star_{h-1}\circ\cdots\circ \pi^\star_0)_{\#} \src$-almost all $x\in\Omega$, the optimal value 
dual variables and the optimal policy respectively satisfy
    \begin{align*}
        \Vstar_h(x) &= \min_{y\in\Omega} \br{\frac{H+1}{2}\norm{x-y}^2 + \Vstar_{h+1}(y)},
        \\
        \pistar_h(x) &= \arg\min_{y\in \Omega}\br{\frac{H+1}{2}\norm{x-y}^2 + \Vstar_{h+1}(y)}.
    \end{align*}
\end{theorem}
\begin{proof}
    For ease of notation, we define $\pibar_h = \pi^\star_{h}\circ\cdots\circ \pi^\star_0$. By strong duality 
(\cref{lem:sd}) and the fact that $(\pibar_{H})_\#\src=\tgt$ by feasibility, we have 
    \begin{align*}
        &\frac{H+1}{2}\sum_{h=0}^H \int \norm{\pi^\star_h(x) - x}^2 d(\pibar_{h-1 \#}\src)(x)\\
        &\qquad\qquad= \int V^\star_0(x)\dd\src(x)-\int V^\star_{H+1}(x)\dd\tgt\\
        &\qquad\qquad= \int V^\star_0(x) \dd\src(x) - \int V^\star_{H+1}(x) d ((\pibar_{H})_\#\src)(x)\\
        &\qquad\qquad= \int (V^\star_0(x) -  V^\star_{H+1}(\pibar_{H}(x)) \dd\src(x) \\
        &\qquad\qquad= \sum_{h=0}^H \int 
\br{V^\star_h(\pibar_{h-1}(x))-V^\star_{h+1}(\pi^\star_h(\pibar_{h-1}(x)))}\dd\src(x)\\
        &\qquad\qquad= \sum_{h=0}^H \int \br{V^\star_h(x)-V^\star_{h+1}(\pi^\star_h(x))}d((\pibar_{h-1})_\#\src)(x) \\ 
        &\qquad\qquad\leq \frac{H+1}{2}\sum_{h=0}^H \int \norm{\pi^\star_h(x)-x}^2 d(\pibar_{h-1})_\#\src(x)
    \end{align*}
    where the inequality uses feasibility of $V^\star$. We must thus for each $h$ have equality almost everywhere: 
    \begin{align*}
        V^\star_h(x)-V^\star_{h+1}(\pi^\star_h(x)) = \frac{H+1}{2}\norm{\pi^\star_h(x)-x}^2.
    \end{align*}
    Hence, by feasibility $V^\star_h(x) \leq \min_{y\in\Omega}\br{\frac{H+1}{2}\norm{x-y}^2+V^\star_{h+1}(y)} \leq 
\frac{H+1}{2}\norm{\pi^\star_h(x)-x}^2 + V^\star_{h+1}(\pi^\star_h(x)) = V^\star_h(x)$ so we must have equality, which 
shows both $V^\star_h(x) = \min_{y\in\Omega}\br{\frac{H+1}{2}\norm{x-y}^2+V^\star_{h+1}(y)}$ and that $\pi^\star_h(x)$ 
is a minimizer. 
\end{proof}

\subsection{Time scaling}\label{app:few-step}

We have seen in the previous sections that the discrete-time dynamic OT problem is equivalent to both the static OT 
problems~(\ref{eq:monge-reprise},\ref{eq:kantorovich-reprise}) and the dynamic OT problem~\eqref{eq:benamou-brenier}. The 
former problem formulation is recovered in our setting by choosing $H=0$. On the other hand, our formulation can be 
intuitively viewed as a time-discretization at scale $1/(H+1)$ of the continuous-time 
Benamou--Brenier formulation~\eqref{eq:benamou-brenier}. As we argue below, these connections can be extended to 
accommodate discrete-time formulations at various time scales as well.

The following result shows that further discretizing the time horizon recovers the same 
optimal transport map but at a finer discretization level (which one may also see by combining \cref{thm:bellman} and 
\cref{thm:det-policies}). The result is a strict generalization of \cref{thm:dual-stat-dyn} (by setting $H'=0$), but we 
state and prove it separately for clarity and to highlight the connection to the Benamou--Brenier formulation.

\begin{lemma}
    Let $H\in\mathbb{Z}_{>0}$ and $H_k\in \mathbb{Z}_{\geq0}$ be such that $H_k+1 = k(H+1)$, and let us define $h_k = 
kh$ for all $h\in\ev{0,1,\dots,H}$. Let $V^{\star,H_k}$ be optimal for the dynamic dual \eqref{eq:dual-lp} with horizon 
$H_k$. Then, there is an optimal solution $V^{\star,H}$ for the dynamic dual \eqref{eq:dual-lp} with horizon $H$ such that 
    \begin{align*}
        V^{\star,H}_h = V^{\star,H_k}_{h_k} \qquad (h=0,\dots,H+1).
    \end{align*}
\end{lemma}

\begin{proof}
For the proof, let us suppose that $V^{\star,H_k}$ is an optimal solution to~\eqref{eq:dual-lp} with horizon $H_k$, and 
let $\wt{V}^{\star,H}$ be defined for all $h=0,\dots,H+1$ as 
\[
 \wt{V}^{\star,H}_h = V^{\star,H_k}_{h_k}.
\]
Below, we will show that $\wt{V}^{\star,H}$ is both feasible and optimal for~\eqref{eq:dual-lp} with horizon $H$, which 
will imply the statement of the lemma.

    \emph{Feasibility:} For any $x,y\in\R^d$ set $x_\ell = (1-\frac \ell k)x + \frac \ell k y$ (for 
$\ell=0,\dots,k$) and notice that by feasibility of $V^{\star,H_k}$, we have 
    \begin{align*}
        \wt{V}^{\star,H}_h(x) =& V^{\star,H_k}_{h_k}(x_0) \\
        \leq& \frac{H_k+1}{2}\norm{x_0-x_1}^2 + V^{\star, H_k}_{h_k+1}(x_1)\\
        \leq& \dots \\
        \leq& \frac{H_k+1}{2}\sum_{\ell=0}^{k-1} \norm{x_\ell-x_{\ell+1}}^2 + V^{\star,H_k}_{k(h+1)}(x_{k})\\
        =& \frac{H_k+1}{2}\sum_{\ell=0}^{k-1} \norm{\frac{1}{k}(x-y)}^2 + \wt{V}^{\star,H}_{h+1}(y)\\
        =& \frac{H_k+1}{2k} \norm{x-y}^2 + \wt{V}^{\star,H}_{h+1}(y)\\
        =& \frac{H+1}{2} \norm{x-y}^2 + \wt{V}^{\star,H}_{h+1}(y),
    \end{align*}
    thus implying that $\wt{V}^{\star,H}$ is feasible for~\eqref{eq:dual-lp} with horizon $H$.

    \emph{Optimality:} Let $\opt_{\text{D},H'}^{\text{dyn}}$ denote the optimal value of~\eqref{eq:dual-lagr} for 
horizon $H'$. We already know from \cref{thm:dual-stat-dyn} that 
    \begin{align}
        \opt_{\text{D},H}^{\text{dyn}} = \opt_{\text{D}} = \opt_{\text{D},H_k}^{\text{dyn}}.  \label{eq:inv-H}
    \end{align}
    Moreover, by construction $\wt{V}^{\star,H}_0 = V^{\star,H_k}_0$ and $\wt{V}^{\star,H}_{H+1} = 
V^{\star,H_k}_{H_k+1}$, so by optimality of $V^{\star,H_k}$, we have 
    \begin{align*}
        \opt_{\text{D},H_k}^{\text{dyn}} =& \int V^{\star,H_k}_0(x)\dd\src(x) - \int 
V^{\star,H_k}_{H_k+1}(x)\dd\tgt(x)\\
        =& \int \wt{V}^{\star,H}_0(x)\dd\src(x) - \int \wt{V}^{\star,H}_{H+1}(x)\dd\tgt(x).
    \end{align*}
    Plugging this into \cref{eq:inv-H} concludes the proof. 
\end{proof}

\paragraph{Equivalence between discrete-time and continuous-time dynamic OT.} 
The above theorem suggests that refining the discretization of the time horizon does not change the structure of the 
optimal value functions and that in the limit, we should recover a potential that corresponds to a straight 
continuous-time transportation path between $\src$ and $\tgt$, as is known to be the case for the Benamou--Brenier 
formulation. We make this precise in the following (\cref{lemma:final-bb}), first introducing several simple auxiliary 
definitions and results (\cref{lemma:aux-bb-1,lemma:aux-bb-2,lemma:aux-bb-3}).

As before, assume that $\supp(\src),~\supp(\tgt)\subset \Omega$ for some compact convex $\Omega\subset\R^d$. For 
$f\in\Cc(\Omega)$, we define the \emph{$s$-scaled Bellman operator} $\Tc_s$ acting on $f$ via 
\begin{align*}
    [\Tc_s f](x) = \min_{y\in\Omega}\br{\frac{\norm{x-y}^2}{2 s} + f(y)} 
\end{align*}
for $s>0$ and $\Tc_0 f = f$. We first give an alternative characterization of the optimal value functions which will be 
useful to establish convergence in the continuous-time limit.

\begin{lemma} \label{lemma:aux-bb-1}
    Let $(W^\star_0,W^\star_1)$ an optimal solution for the static dual problem \eqref{eq:kantorovich-dual}, chosen such 
that $W_0^\star = \Tc_1 W_1^\star$ on all of $\Omega$. Further, set $t(h,H) = h/(H+1)$ for $h=0,\dots,H+1$ and set 
$V^{\star,H}_h = \Tc_{1-t(h,H)} W_1^\star$. Then, $V^{\star,H}$ is optimal for~\eqref{eq:dual-lp} with horizon $H$.
\end{lemma}

\begin{proof}
    First, note that we can indeed choose $W^\star$ such that $W_0^\star = \Tc_1 W_1^\star$ on all of $\Omega$: By optimality, we have $W_0^\star = \Tc_1 W_1^\star$ on $\supp(\src)$, and by continuity of $W_0^\star$ and $\Tc_1 W_1^\star$, we can choose a continuous extension to all of $\Omega$ preserving this equality. 

    Next, it is straightforward to verify that for $s, s' \geq 0$, we have $\Tc_{s'} \circ \Tc_s = \Tc_{s+s'}$ (by convexity of $\Omega$). Using this, we see that for any $h=0,\dots,H$ and $x\in\Omega$, we have 
    \begin{align*}
        V^{\star,H}_h(x) =& \Tc_{1-h/(H+1)} W_1^\star(x)=\Tc_{1/(H+1)} \Tc_{1-(h+1)/(H+1)} W_1^\star(x) = \Tc_{1/(H+1)} V^{\star,H}_{h+1}(x),
    \end{align*}
    that is, $V^{\star,H}(x) = \min_{y\in\Omega}\br{\frac{H+1}{2}\norm{x-y}^2 + V^{\star,H}_{h+1}(y)}$. In addition, by 
compactness of $\Omega$ and continuity of $W_1^\star$, the value functions $V^{\star,H}_h$ are continuous and thus 
feasible. Finally, $V^{\star,H}_0 = \Tc_1 W_1^\star = W_0^\star$ and $V^{\star,H}_{H+1} = W_1^\star$, so the objective 
value of $V^{\star,H}$ in the dynamic dual is the same as that of $W^\star$ in the static dual, which is optimal by 
\cref{thm:dual-stat-dyn}.
\end{proof}

We next show that the optimal value functions converge to their continuous-time counterparts as we refine the discretization appropriately. 

\begin{lemma} \label{lemma:aux-bb-2}
    Let $t\in[0,1]$ and consider $(h_k, H_k)_{k\geq1}$ such that for all $k$, we have $h_k \in \{0, \dots, H_k\}$ and $\frac{h_k}{H_k+1} \to t$ as $k \to \infty$. Then for all $x\in\Omega$, $V^{\star,H_k}_{h_k}(x)$ from \cref{lemma:aux-bb-1} converges to $\Tc_{1-t} W_1^\star(x)$ as $k\to \infty$.
\end{lemma}

\begin{proof}
    We first claim that for any $x\in\Omega$, the map $s \mapsto \Tc_s W_1^\star(x)$ is continuous on $[0,1]$. For 
$s,s'>0$, we have $\abs{\Tc_s W_1^\star(x) - \Tc_{s'} W_1^\star(x)} \leq 
\frac{\text{diam}(\Omega)^2}{2\min\{s,s'\}^2}\abs{s-s'}$, showing continuity at every $s>0$. For $s=0$, consider a 
sequence $s_n \to 0$ ($n\to\infty$) and notice that by boundedness of $W^\star_1(x)$, any minimizer $y_{s_n}$ in the 
definition of $\Tc_{s_n}W^\star_1(x)$ must satisfy $\norm{x-y_{s_n}} \to 0$. Hence by continuity of $W_1^\star$, we have 
$\Tc_{s_n} W_1^\star(x) = \frac{\norm{x-y_{s_n}}^2}{2 s_n} + W_1^\star(y_{s_n}) \geq W_1^\star(y_{s_n}) \to 
W_1^\star(x)$ as $n \to \infty$. On the other hand, we have $W_1^\star(y_{s_n}) + \frac{\norm{x-y_{s_n}}^2}{2 s_n} = \min_{y\in\Omega} 
\br{\frac{\norm{x-y}^2}{2s_n} + W_1^\star(y)} \leq W_1^\star(x)$. Hence $\limsup_{n\to\infty} \Tc_{s_n} W_1^\star(x) 
\leq W_1^\star(x) \leq \liminf_{n\to\infty} \Tc_{s_n} W_1^\star(x)$ and thus $\Tc_{s_n} W_1^\star(x) \to W_1^\star(x) = 
\Tc_{0} W_1^\star(x)$ as $n\to\infty$, showing continuity at $s=0$.

    Hence, we find that for any $x\in\Omega$, 
    \begin{align*}
        V^{\star,H_k}_{h_k}(x) =& \Tc_{1-h_k/(H_k+1)} W_1^\star(x) \to \Tc_{1-t} W_1^\star(x) \qquad (k\to\infty)
    \end{align*}
    as claimed. 
\end{proof}

\begin{lemma} \label{lemma:aux-bb-3} 
    Let $M^\star$ be a solution to the static OT problem \eqref{eq:monge-reprise}. Let $x\in\supp(\src)$ and $t\in[0,1)$ 
and suppose that $x_t=(1-t)x + t M^\star(x) \in \text{int}(\Omega)$. Then $\Tc_{1-t}W_1^\star$ is differentiable at 
$x_t$ with gradient $\nabla_x \Tc_{1-t} W_1^\star(x_t) = x - M^\star(x)$.
\end{lemma}

\begin{proof}
    Let $x\in\supp(\src)$ and set $x_t = (1-t)x + t M^\star(x)$. Recall that $\Tc_{1-t} W_1^\star(x_t) = 
\min_{y\in\Omega}\br{\frac{1}{2(1-t)}\norm{x_t-y}^2 + W_1^\star(y)}$. Now for all $y\in\Omega$, we have $W_1^\star(y) 
\geq W_0^\star(x) - \frac{1}{2}\norm{x-y}^2$ by feasibility of $W^\star$. Hence 
    \begin{align*}
        \frac{1}{2(1-t)}\norm{x_t-y}^2 + W_1^\star(y) \geq& W_0^\star(x) + \frac{1}{2(1-t)}\norm{x_t-y}^2 - \frac{1}{2}\norm{x-y}^2 \\
        =& W_0^\star(x) - \frac{t}{2}\norm{M^\star(x)-x}^2 + \frac{t}{2(1-t)}\norm{y-M^\star(x)}^2,
    \end{align*}
    where the last equality simply follows from noticing $x_t-y=-((y-M^\star(x)) + (1-t)(M^\star(x)-x))$ and 
$x-y=-((M^\star(x)-x)+(y-M^\star(x)))$ and expanding the squares. The final term is non-negative and zero if and only 
if $y=M^\star(x)$. Furthermore, the above inequality is an equality in this case by optimality of $W^\star$ (c.f. 
\cref{thm:det-policies}). Hence $M^\star(x)$ is the unique minimizer in the definition of $\Tc_{1-t}W_1^\star(x_t)$. 
Hence, by Danskin's theorem, $\Tc_{1-t}W_1^\star$ is differentiable at $x_t$ if $x_t\in \text{int}(\Omega)$, with 
    \begin{align*}
        \nabla_x \Tc_{1-t} W_1^\star(x_t) = \frac{1}{1-t}(x_t-M^\star(x)) = x - M^\star(x).
    \end{align*} 
    This concludes the proof.
\end{proof}

From \cref{lemma:aux-bb-2,lemma:aux-bb-3} we can immediately deduce the following.

\begin{corollary} \label{lemma:final-bb}
    Let $(V^{\star,H_k}_{h_k})_{k\geq1}$ be the optimal value functions for \eqref{eq:dual-lp} from 
\cref{lemma:aux-bb-2} such that $h_k/(H_k+1) \to t \in [0,1)$ as $k\to\infty$. Let $M^\star$ be an optimal solution for 
the static OT problem \eqref{eq:monge-reprise}. Then, for any $x\in\supp(\src)$ such that $x_t=(1-t)x + t M^\star(x) 
\in \text{int}(\Omega)$, we have $V^{\star,H_k}_{h_k}(x_t) \to V^\infty_t(x_t)$ ($k\to\infty$) for a differentiable 
function $V^\infty_t$ with
    \begin{align*}
        \nabla V^\infty_t(x_t) = x - M^\star(x).
    \end{align*}
\end{corollary}

As shown by \citet{benamou2000computational}, the optimal solution to the 
dynamic OT problem~\eqref{eq:benamou-brenier} is exactly given by the vector field $u_t$ mapping $x_t$ to 
$M^\star(x)-x$ in the notation of \cref{lemma:final-bb}. Thus, the corollary above implies that the optimal value 
functions of the dynamic dual problem \eqref{eq:dual-lagr} converge (pointwise) to a well-defined limit that is differentiable, and whose gradient corresponds to the (negative of the) optimal vector field in the Benamou--Brenier formulation---thus connecting the solutions of the discrete-time and continuous-time dynamic OT problems.

\section{Further implementation details}\label{app:implementation-details}
We now discuss the handful of additional details about the practical implementation of our training and 
generation methods that did not fit into Sections~\ref{sec:VDT_alg} and~\ref{sec:prediction}.

\subsection{Details of VDT training}
The two main components of our VDT training algorithm are the primal and dual updates. These are performed in our 
practical implementation as follows.

\paragraph{Primal updates.} The updates are performed according to Equation~\eqref{eq:particle_update} for a total of 
$K$ steps with a constant stepsize. In addition to the gradient updates, we have found it useful to add a small amount 
of Gaussian noise to each update, partly to fight the non-convexity of the primal objective. Indeed, note that the 
particles move along gradients of a function that is inherently non-convex with respect to the state.
Additionally, due to the well-known connection between such noisy Wasserstein gradient descent steps and 
entropy-regularized sampling \citep{JKO98,WT11}, this step can be seen as implicitly adding a small amount of entropy 
regularization to the primal objective in terms of $\mu$. We note though that the type of entropy regularization that 
this step adds to our OT problem is distinct from the regularization appearing in the famous Schr\"odinger bridge 
problem \citep{Sch31} that serves as the basis of many modern generative models. Instead, our regularization acts on 
the marginal distributions $\mu_h$ as opposed to the joint distributions of trajectories, which is known to achieve 
distinct regularization effects in Markov decision processes---see \citet{NJG17} for more details.

For initializing the particles, our preferred choice is to compute an optimal-transport coupling of the sample 
$(\Xsrc(i),\Xtgt(i))_{i=1}^b$ in the minibatch. For this purpose, since the two particle clouds are of identical size, we can 
simply employ a standard solver for deterministically matching the two sets. In our implementation, we use the 
classic Hungarian algorithm for solving this problem (as implemented in the \texttt{linear\_sum\_assignment} function 
of SciPy \citep{scipy}).

\paragraph{Dual updates.} In all our experiments, we use the stochastic gradients computed in 
Equation~\eqref{eq:dual_update} to update parameters using the Adam optimization algorithm \cite{KB14}.

\subsection{Details of VDT prediction}\label{app:vdt_pred_details}
Generating samples using VDT policies is straightforward via Algorithm~\ref{alg:VDT_prediction}. We provide a few 
additional details about the enhancements of this subroutine below.

\paragraph{Few-step generation.} For the purpose of few-step generation, we parametrize our value functions to take 
inputs in the unit interval, by normalizing the value of $h$ as $t_h = \frac{h}{H+1}$. This is straightforward to 
incorporate in both the training and generation subroutines (and in fact, Algorithm~\ref{alg:VDT_prediction} already 
adopts this notational convention). Choosing the prediction horizon $\Htest$ can be seen to be analogous to choosing 
the hyperparameters of the numerical integration subroutine required by all continuous time models (such as diffusion 
and flow-based models). In light of our results in Section~\ref{app:few-step}, our choice can be seen as a simple 
forward Euler discretization of the continuous-time dynamic OT solution path, which suggests that one can possibly 
incorporate other advanced ideas from numerical integration into our generation subroutine. We did not pursue this 
direction in the present paper.

\paragraph{Classifier-free guidance.}
Classifier-free guidance (CFG,\citep{HS22}) is a popular technique for improving the sample quality of conditional 
generative models. At a high level, a small fraction of samples is used for unconditional training, and at inference 
time the sampling direction of the unconditional model is corrected by a scaled version of the difference between 
proposed directions of the conditional and unconditional model. This idea can be easily incorporated into our setup by 
adding the same correction term to the VDT generation process, in particular by changing the generation steps for a \emph{guidance scale} 
$\alpha > 0$ as
\[
 X_{h+1} = X_h - \frac{1}{H+1} \pa{\alpha \nabla_x V^{\text{conditional}}_h(X_h) + (1-\alpha) \nabla_x 
V^{\text{unconditional}}_h(X_h)}.
\]
In particular, no changes to conditional training are needed except for not labeling a small fraction 
($p_{\text{uncond}}=0.2$ in our experiments) of training samples.

\section{Experiments}
This section provides further details on our experiments, as well as some additional results that were omitted from 
the main text. For each setting, we will describe all the hyperparameters (neural network architecture for the value 
functions and optimizer settings), the data sets and the evaluation procedures. 

All experiments were executed on two academic-size clusters with GPU nodes. In particular, 2 NVIDIA H100 GPUs ($\sim$ 
96GB of available GPU memory) from one HPC cluster and 2 NVIDIA L40s from another one ($\sim$45GB of available GPU 
memory). Each experiment was run on a single GPU, and the more compute-intensive experiments were performed on a single 
NVIDIA H100 (whose memory was never used to completion by our experiments).

\subsection{2D experiments}\label{app:more-2d-experiments}
We implement our method using a three-layer neural network with 64 units per layer to represent a value function, 
and add a $32$-dimensional time embedding layer to encode the time input $h$. We tune the remaining hyperparameters as 
follows. For training, we set the batch size as $B = 100$ and perform $T = 20,000$ dual updates using Adam with 
hyperparameters $\eta_0 = 10^{-4}$, $\beta_1 = 0.9$, $\beta_2 = 0.999$ and $\epsilon = 10^{-8}$. Within the inner loop, 
we perform $K = 5$ updates with constant stepsize $\gamma = 0.5$ and a noise scale of $\alpha = 10^{-3}$.

We use the following data sets, all taken from \citet{TMF+23b} (with some additional details extracted from 
\citet{SdBCD23} and the associated code repositories):
\begin{itemize}
 \item ``moons'': We use the standard ``two moons'' data set as implemented in scikit-learn \citep{scikit-learn}, with 
noise parameter set to $0.05$ and each datapoint $x$ transformed as $3x - 1$ for scaling purposes.
 \item ``scurve'': We use the standard ``scurve'' data set as implemented in scikit-learn \citep{scikit-learn}, with 
noise parameter set to $0.05$ and each datapoint $x$ transformed as $1.5 x$ for scaling purposes.
 \item ``8gauss'': We set up eight Gaussian distributions with means evenly spaced on a circle of radius $5$ centered 
at the origin. The standard deviation of each Gaussian is set to $0.1$.
 \item ``moons-8gauss'': We use the ``moons'' and ``8gauss'' datasets described above, scaled up with a factor $2$.
\end{itemize}

For evaluation, we closely follow the setup of \citet{SdBCD23}. We set the sample size as $n = 10,000$ and repeat each 
experiment $5$ times and report the means along with the standard deviations in Table~\ref{table:2d_result_full}. For 
computing the Wasserstein distances (necessary both for evaluating generation quality and computing the oracle solution 
for the path energies), we use the Hungarian algorithm over a fresh sample batch of $10,000$ samples. 

All competing methods are based on integrating continuous-time ODEs or SDEs, and thus require a numerical method for 
approximating the learned transport maps. The numbers we report are based on a forward Euler discretization for $20$ 
steps. Our method does not explicitly require tuning such a hyperparameter, although the role of $\Htest$ for few-step 
generation is similar. For a well-trained model, setting $\Htest = 10$ worked nearly as well in our experiments as 
using the default choice $\Htest = H = 100$.

\begin{figure}[h]
\centering
\setlength{\tabcolsep}{5pt}
\renewcommand\arraystretch{0.8}
\vspace{-0.3cm}
\scalebox{0.63}{
    \begin{tabular}{ccccc}
    \toprule 
     & \multicolumn{4}{c}{\textit{Wasserstein-2 distance from target}}\tabularnewline
    \cmidrule{2-5} \cmidrule{3-5} \cmidrule{4-5} \cmidrule{5-5} 
    \textit{Dataset} & moons & scurve & 8gaussians & moons-8gaussians\tabularnewline
    \midrule[1pt]
100 step VDT+ & \textbf{0.131\small{\textpm  0.034}} &  \textbf{0.120\small{\textpm  0.013}}& 0.435\small{\textpm  
0.123} 
& \textbf{0.652 \small{\textpm  0.151}} \tabularnewline
50 step VDT+ & \textbf{0.131\small{\textpm  0.034} }&  \textbf{0.120\small{\textpm  0.012}}& 0.434 \small{\textpm  
0.123} 
& \textbf{0.646 \small{\textpm  0.148} }\tabularnewline
20 step VDT+ & \textbf{0.131\small{\textpm  0.032}} &  \textbf{0.122\small{\textpm  0.013}}& 0.430\small{\textpm  
0.115} 
& \textbf{0.634 \small{\textpm  0.131}} \tabularnewline
10 step VDT+ &  \textbf{0.132 \small{\textpm  0.029}} &  0.125\small{\textpm  0.013}& 0.424\small{\textpm  0.111} 
& \textbf{0.626 \small{\textpm  0.097 }} \tabularnewline
5 step VDT+ & 0.135 \small{\textpm  0.022} &  0.136\small{\textpm  0.013}& 0.430 \small{\textpm  0.089} 
& \textbf{0.623\small{\textpm  0.065}}\tabularnewline
2 step VDT+ & 0.160 \small{\textpm  0.008} &  0.178 \small{\textpm  0.012}& 0.530\small{\textpm  0.050} 
& 0.783\small{\textpm  0.046}\tabularnewline
1 step VDT+ & 0.229 \small{\textpm  0.008} &  0.262\small{\textpm  0.004}& 0.809\small{\textpm  0.025} 
& 1.365\small{\textpm  0.039}\tabularnewline
\midrule
100 step VDT & 0.219\small{\textpm  0.086} &  0.208\small{\textpm  0.016}& 0.547\small{\textpm  0.120} 
& 1.205 \small{\textpm  0.125} \tabularnewline
50 step VDT & 0.221 \small{\textpm  0.085} &  0.205\small{\textpm  0.015}& 0.546\small{\textpm  0.110} 
& 1.229 \small{\textpm  0.1} \tabularnewline
20 step VDT & 0.242 \small{\textpm  0.072} &  0.213\small{\textpm  0.030}& 0.545 \small{\textpm  0.102} 
& 1.283\small{\textpm  0.062} \tabularnewline
10 step VDT & 0.307\small{\textpm  0.035} &  0.273\small{\textpm  0.058}& 0.565\small{\textpm  0.073 } 
& 1.360 \small{\textpm  0.077 } \tabularnewline
5 step VDT & 0.45 \small{\textpm  0.046} &  0.457\small{\textpm  0.052 }& 0.724 \small{\textpm  0.098} 
& 1.553\small{\textpm  0.099}\tabularnewline
2 step VDT & 0.8 \small{\textpm  0.024} &  0.958 \small{\textpm  0.084}& 1.504\small{\textpm  0.209} 
& 2.757 \small{\textpm  0.043 }\tabularnewline
1 step VDT & 1.497 \small{\textpm  0.046} &  1.803\small{\textpm  0.048}& 4.152 \small{\textpm  0.099} 
& 6.498 \small{\textpm  0.058}\tabularnewline
    \midrule
    SF$^2$M+ & \textbf{0.124\small{\textpm  0.023}} & \textbf{0.128\small{\textpm  0.005}} & 
\textbf{0.275\small{\textpm  0.058}} & 
{0.726\small{\textpm  0.137}}\tabularnewline
    SF$^2$M & {0.185\small{\textpm  0.028}} & {0.201\small{\textpm  0.062}} & {0.393\small{\textpm  0.054}} & 
{1.482\small{\textpm  0.151}}\tabularnewline
     DSBM-IPF & 0.140\textpm 0.006 & 0.140\textpm 0.024 & 0.315\textpm 0.079 & \textit{0.812\textpm 
 0.092}\tabularnewline
    DSBM-IMF++ & \textbf{0.123\small{\textpm  0.014}} & \textbf{0.130\small{\textpm  0.025}} & 
\textbf{0.276\small{\textpm  0.030}} & 0.802\small{\textpm  0.172} \tabularnewline
    DSBM-IMF & 0.144\small{\textpm  0.024 }& 0.145\small{\textpm  0.037 }& 0.338\small{\textpm  0.091 }& 
0.838\small{\textpm  0.098}\tabularnewline
    DSB & 0.190\textpm 0.049 & 0.272\textpm 0.065 & 0.411\textpm 0.084 & 0.987\textpm 0.324\tabularnewline
    SB-CFM & \textit{0.129\textpm 0.024} & \textit{0.136\textpm 0.030} & \textbf{0.238\textpm 0.044} & 0.843\textpm 
0.079\tabularnewline
    FM & 0.212\textpm 0.025 & 0.161\textpm 0.033 & 0.351\textpm 0.066 & -\tabularnewline
    CFM & 0.215\textpm 0.028 & 0.171\textpm 0.023 & 0.370\textpm 0.049 & \textit{1.285\textpm 0.314}\tabularnewline
    OT-CFM+ & \textbf{0.130\small{\textpm  0.016}} & {0.144\small{\textpm  0.028}} & {0.303\small{\textpm  
0.043}} & \textbf{0.601\small{\textpm  0.027}}\tabularnewline
RF & 0.283\small{\textpm  0.045}& {0.345\small{\textpm  0.079}} & {0.421\small{\textpm  0.071}} & 
1.525\small{\textpm 0.330}\tabularnewline
    \midrule
    oracle & - & - & - & 
- \tabularnewline
    \bottomrule
    \end{tabular}
}
\scalebox{0.63}{
    \begin{tabular}{ccccc}
    \toprule 
    \multicolumn{4}{c}{\textit{Path energy}}\tabularnewline
    \midrule
    moons & scurve & 8gaussians & moons-8gaussians\tabularnewline
    \midrule[1pt] 
1.238 \small{\textpm  0.059} &  \textbf{1.629\small{\textpm  0.027}}& \textbf{14.386\small{\textpm  0.135}} 
& \textbf{30.444 \small{\textpm  0.311}}\tabularnewline
1.238 \small{\textpm  0.059} &  \textbf{1.629 \small{\textpm  0.027}}&  \textbf{14.380 \small{\textpm  0.136} }
& \textbf{30.400 \small{\textpm  0.310}}+\tabularnewline
1.237 \small{\textpm  0.059} &  \textbf{1.627\small{\textpm  0.026}}& \textbf{14.354 \small{\textpm  0.138} }
&  \textbf{30.258  \small{\textpm  0.301}}\tabularnewline
1.236 \small{\textpm  0.058} &  \textbf{1.623\small{\textpm  0.026}}& \textbf{14.290\small{\textpm  0.148}}
& \textbf{29.989\small{\textpm  0.278}}\tabularnewline
1.228\small{\textpm  0.057} &  1.606\small{\textpm  0.026}& 14.067\small{\textpm  0.152} 
& 29.376 \small{\textpm  0.231}\tabularnewline
\textbf{1.184 \small{\textpm  0.062}} &  1.531\small{\textpm  0.028}& 13.025\small{\textpm  0.145} 
& 27.194\small{\textpm  0.177}\tabularnewline
1.066 \small{\textpm  0.073} &  1.358 \small{\textpm  0.026}& 10.817 \small{\textpm  0.254} 
& 22.679 \small{\textpm  0.333}\tabularnewline
    \midrule
2.416 \small{\textpm  0.096} &  3.051\small{\textpm  0.084}& 17.799\small{\textpm  0.328} 
& 73.001 \small{\textpm  2.228}\tabularnewline
2.403 \small{\textpm  0.096} &  3.043\small{\textpm  0.082}&  17.822 \small{\textpm  0.324} 
& 73.487 \small{\textpm  2.249}\tabularnewline
2.366 \small{\textpm  0.092} &  3.016 \small{\textpm  0.072}& 17.886 \small{\textpm  0.326} 
&  74.677  \small{\textpm  2.099}\tabularnewline
2.272 \small{\textpm  0.073} &  2.955 \small{\textpm  0.07}& 18.046\small{\textpm  0.332} 
& 76.488 \small{\textpm  2.682}\tabularnewline
2.07\small{\textpm  0.098} &  2.761 \small{\textpm  0.062}& 18.282\small{\textpm  0.361} 
& 78.488 \small{\textpm  3.955 }\tabularnewline
1.324 \small{\textpm  0.056} &  1.432 \small{\textpm  0.094}& 18.115\small{\textpm  2.170} 
& 66.712 \small{\textpm  3.355}\tabularnewline
1.154 \small{\textpm  0.061} &  0.82 \small{\textpm  0.036}& 0.419 \small{\textpm  0.072} 
& 43.230 \small{\textpm  4.396}\tabularnewline
    \midrule 
    \textbf{1.183\small{\textpm  0.043}} & \textbf{1.686\small{\textpm  0.039}} & {14.66\small{\textpm  0.173}} & 
\textbf{31.36\small{\textpm 0.930}}\tabularnewline
{2.08 \small{\textpm  0.146}} & {3.01\small{\textpm  0.173}} & {16.74\small{\textpm  0.274}} & {107.3\small{\textpm  
9.695}}\tabularnewline
     1.598\small{\textpm 0.034 }& \textit{2.110\small{\textpm 0.059}} & 14.91\small{\textpm 0.310} & 42.16\small{\textpm
1.026}\tabularnewline
{1.594\small{\textpm  0.043}} & 2.116\small{\textpm  0.018} & {14.88\small{\textpm  0.252}} & {41.09\small{\textpm  
1.206} }
\tabularnewline
    {1.580\small{\textpm  0.036}} & {2.092\small{\textpm  0.053}} & \textbf{14.81\small{\textpm 0.255}} & 
{41.00\small{\textpm 1.495}}\tabularnewline
    - & - & - & - \tabularnewline
    1.649\small{\textpm 0.035} & 2.144\small{\textpm 0.044 }& 15.08\small{\textpm 0.209} & 45.69\small{\textpm 0.661 }
\tabularnewline
    2.227\small{\textpm 0.056} & 2.950\small{\textpm 0.074} & 18.12\small{\textpm 0.416 }& -\tabularnewline
    2.391\small{\textpm 0.043} & 3.071\small{\textpm 0.026}& 18.00\small{\textpm 0.090 }& 116.5\small{\textpm 
2.633}\tabularnewline
    \textbf{1.216\small{\textpm  0.01}} & \textbf{1.675\small{\textpm  0.019}} & {14.88\small{\textpm  0.389}} & 
\textbf{30.47\small{\textpm 0.300}}\tabularnewline 
{1.269\small{\textpm  0.088}} & {1.793\small{\textpm  0.107}} & {15.06\small{\textpm  0.447}} & 
{36.11\small{\textpm 2.701}}\tabularnewline 
\midrule
{1.123\small{\textpm  0.01}} & {1.631\small{\textpm  0.03}} & {14.43\small{\textpm  0.045}} & {30.02\small{\textpm  
0.076}}\tabularnewline
    \bottomrule
    \end{tabular}
}
\captionof{table}{Sampling quality as measured by Wasserstein-2 distance to target and path energy for the 2D 
experiments. We report the mean and standard deviation of 5 independent runs, with best results highlighted in bold 
(somewhat generously, to help guide the attention of the reader). Results for all competing methods taken from 
\citet{TMF+23b}. Methods marked with ``+'' use an OT coupling for forming
minibatches, and methods using a full OT plan over the entire data set for initialization are marked with ``++''.}
\label{table:2d_result_full}
\end{figure}

\subsection{MNIST experiments}\label{app:more-mnist-experiments}
In this set of experiments, we work with the classic MNIST data set \citep{MNIST}, comprised of $28\times28$ greyscale 
images of handwritten digits, represented as vectors with dimension $d = 784$. For all experiments, we use a 
convolutional neural network with $832,289 $ trainable parameters. For VDT training, we set the batch size as $B = 128$ 
and perform $T = 20,000$ dual updates using AdamW with 
hyperparameters $\eta_0 = 10^{-3}$, $\beta_1 = 0.0$, $\beta_2 = 0.999$ and $\epsilon = 10^{-8}$, and a weight decay 
parameter of $10^{-2}$. Within the inner loop, we perform $K = 5$ updates with constant stepsize $\gamma = 0.25$ and a 
noise scale of $\alpha = 10^{-3}$.

\paragraph{Paired deblurring.} For the deblurring experiments, we downsample the original $28\times28$ MNIST images to 
dimension $m\times m$ and linearly upsample them back to the original full dimension. The resulting blurred image $\Xsrc(i)$ 
is paired with its orginal counterpart $\Xtgt(i)$ during VDT training. We train a model for $m=6,10,14$ each and evaluate 
the result on a holdout set of unseen blurred MNIST images. In addition to the results shown in the 
main paper, further results are presented on Figures~\ref{fig:more-deblur1}--\ref{fig:more-deblur3}.

\begin{figure}[t]
\centering
\includegraphics[width=.4\linewidth]{
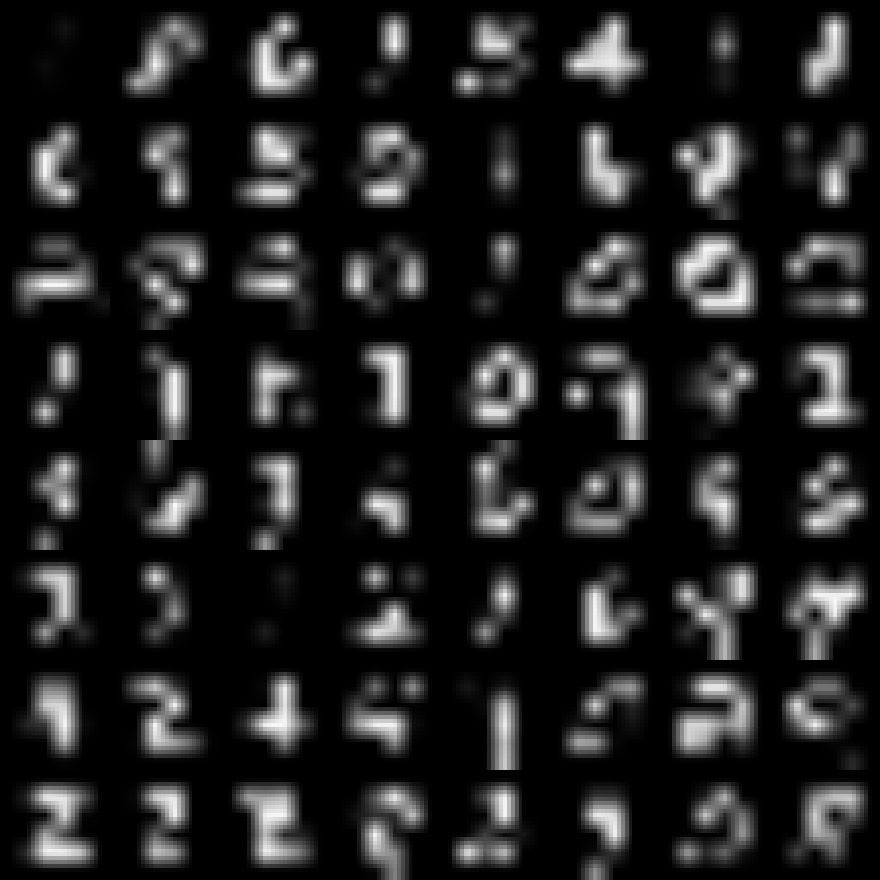} 
\hspace{.5cm}
\includegraphics[width=.4\linewidth]{
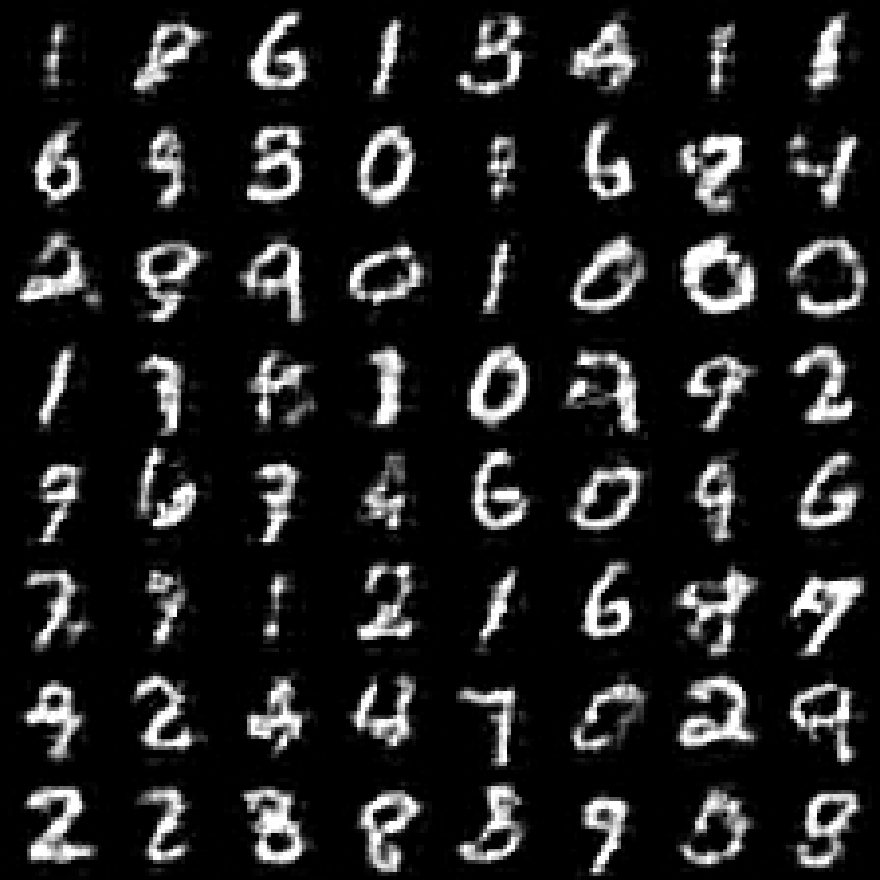}
\caption{Downscaled images and their deblurred counterparts produced by VDT, $m=6$.}\label{fig:more-deblur1}
\end{figure}

\begin{figure}[t]
\centering
\includegraphics[width=.4\linewidth]{
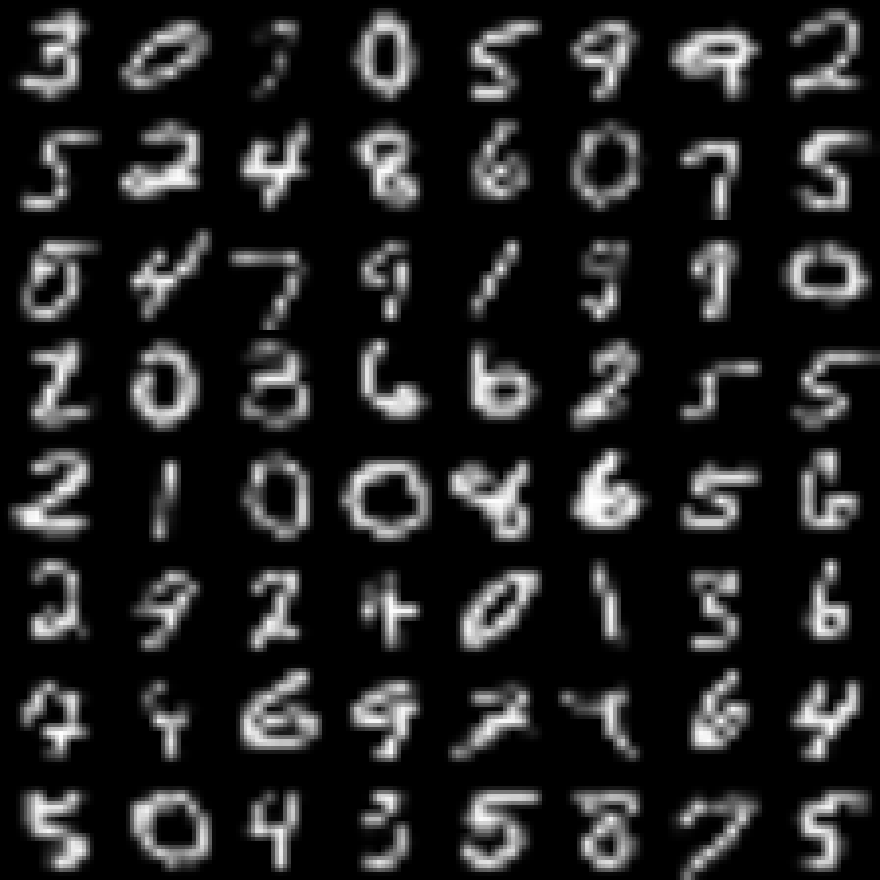} 
\hspace{.5cm}
\includegraphics[width=.4\linewidth]{
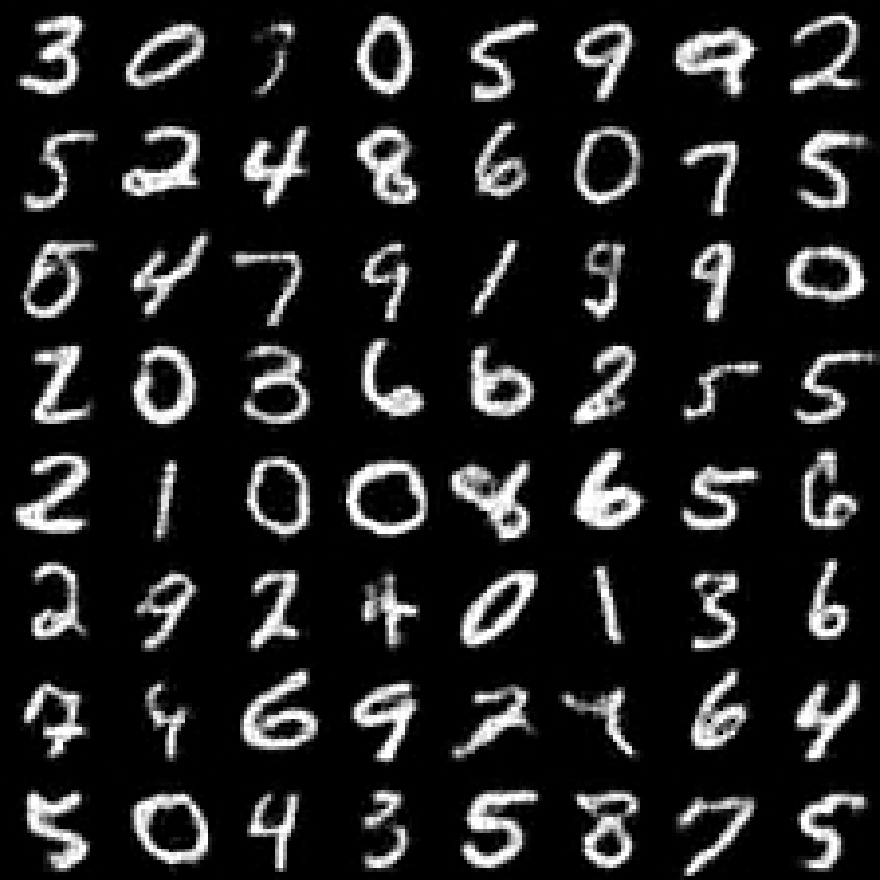}
\caption{Downscaled images and their deblurred counterparts produced by VDT, $m=10$.}\label{fig:more-deblur2}
\end{figure}

\begin{figure}[t!]
\centering
\includegraphics[width=.4\linewidth]{
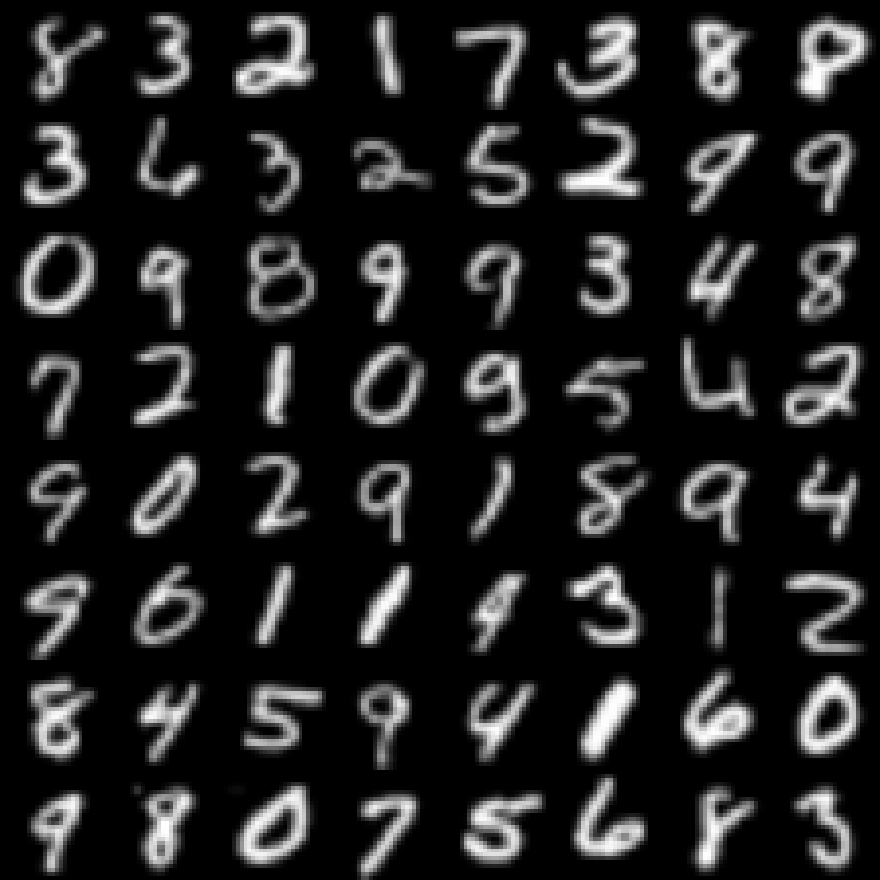} 
\hspace{.5cm}
\includegraphics[width=.4\linewidth]{
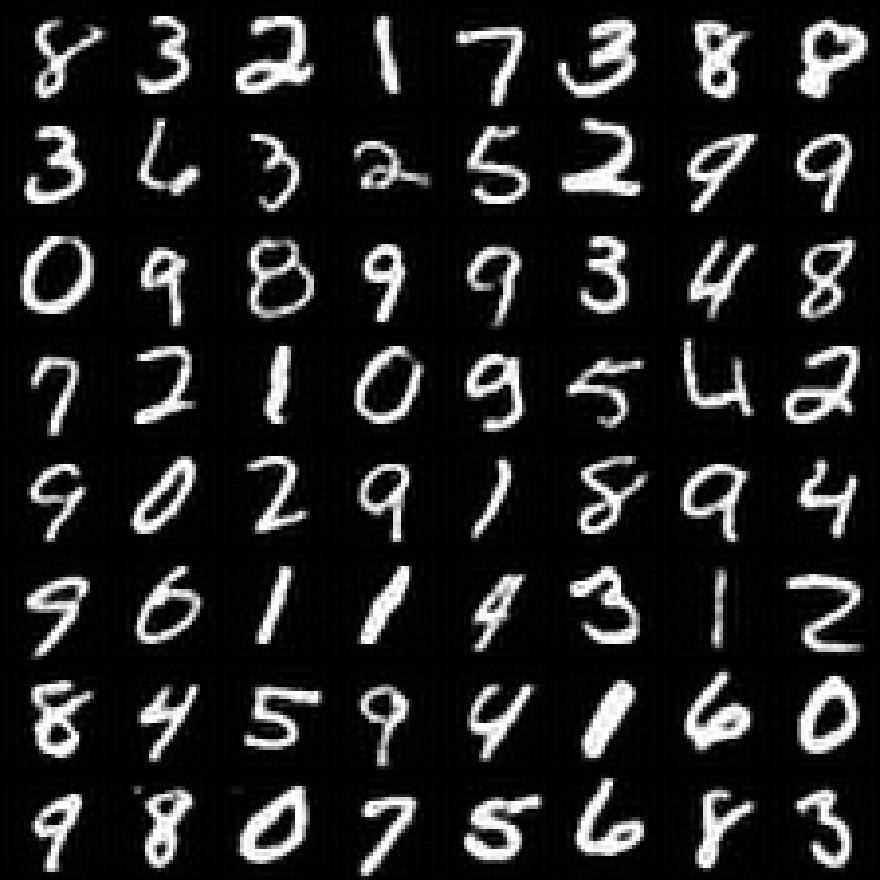}
\caption{Downscaled images and their deblurred counterparts produced by VDT, $m=14$.}\label{fig:more-deblur3}
\end{figure}

\paragraph{Unpaired letter to digit translation.} We next consider a task without naturally available pairings, namely 
the one of translating handwritten letters (`a'-`e' and `A'-`E' from EMNIST \citep{EMNIST}) to handwritten digits (0-9 
from MNIST). We train a VDT model on unpaired minibatches of letters and digits, and evaluate the result on a holdout 
set of unseen letters in \cref{fig:letter_digit}. While the sharpness of the generated digits could likely be improved 
through heavier training, they clearly exhibit the desired resemblance of the source letters. In short, VDT is able to 
learn a nontrivial relation between the domains without explicit supervision. Notably, this relation naturally allows 
for generation along the reverse path (i.e. from digits to letters) without any additional training, as we remarked 
previously. Additionally, the straightness of the VDT paths allow for few-step translation between the two 
distribution, with some example outputs shown in the main text. We show some additional examples in 
Figures~\ref{fig:more-unpaired1} and~\ref{fig:more-unpaired2}.

\begin{figure}
\centering
\includegraphics[width=.4\linewidth]{
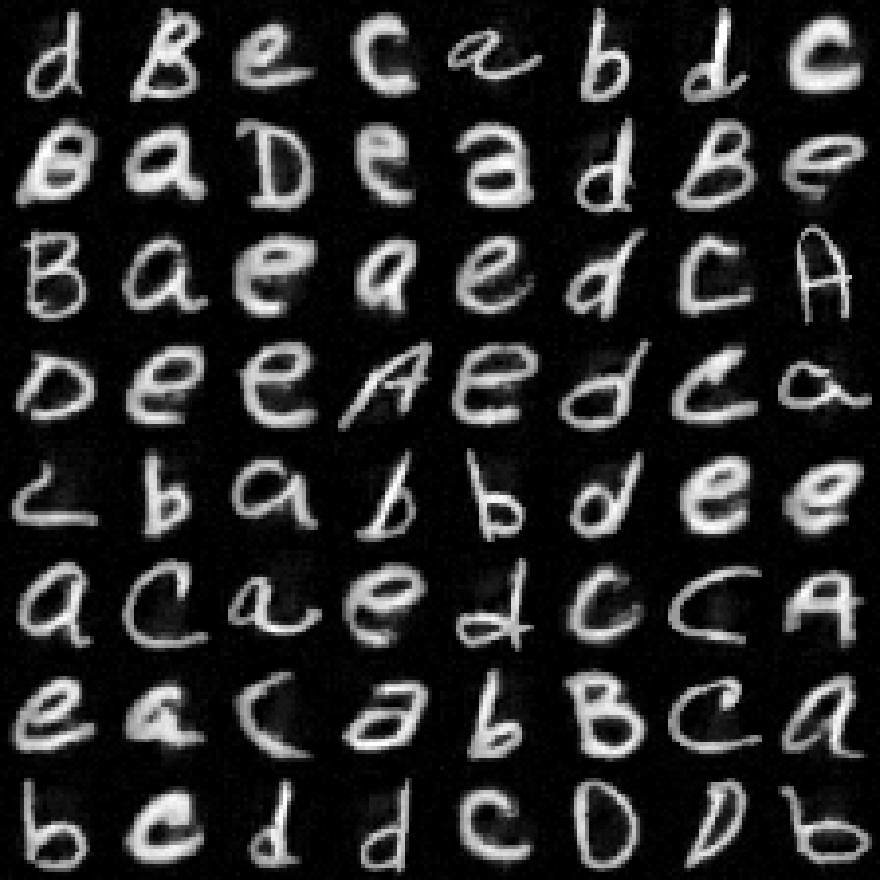} 
\hspace{.5cm}
\includegraphics[width=.4\linewidth]{
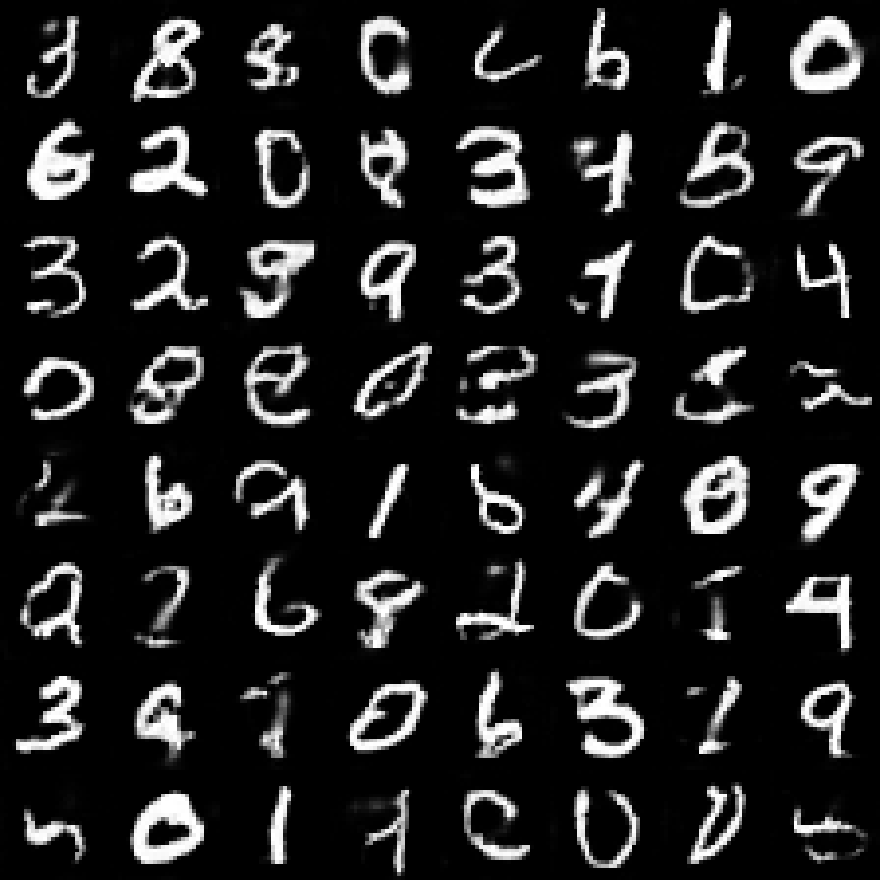}
\caption{Forward sampling from a  VDT policy from EMNIST to MNIST.}\label{fig:more-unpaired1}
\end{figure}

\begin{figure}
\centering
\includegraphics[width=.4\linewidth]{
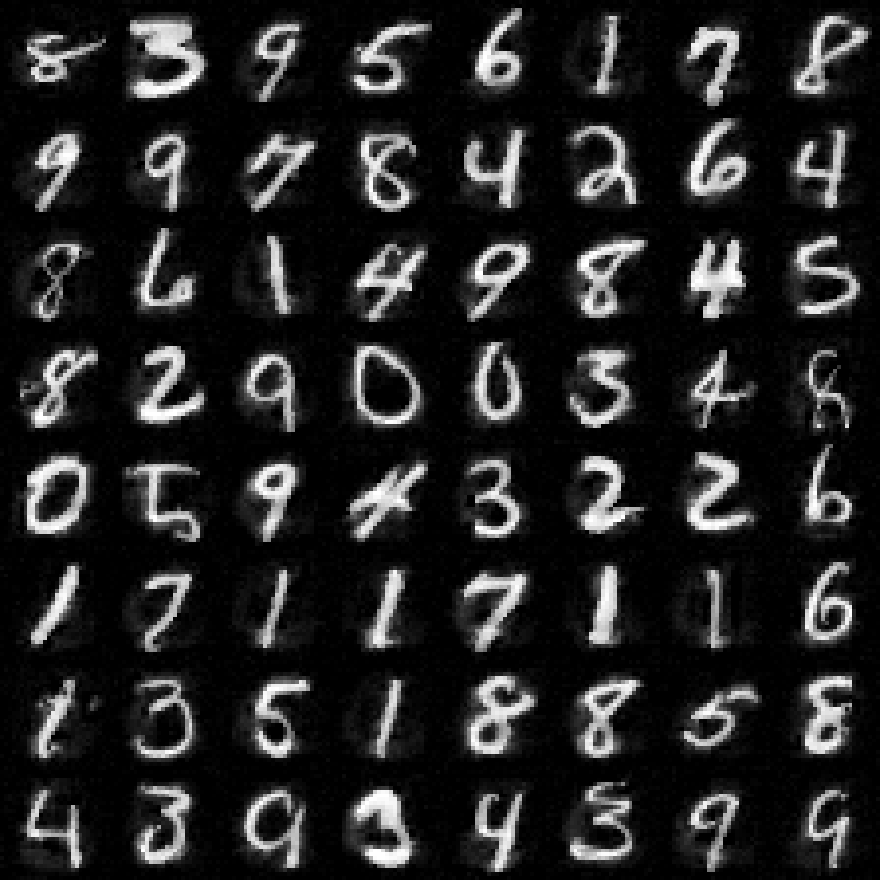} 
\hspace{.5cm}
\includegraphics[width=.4\linewidth]{
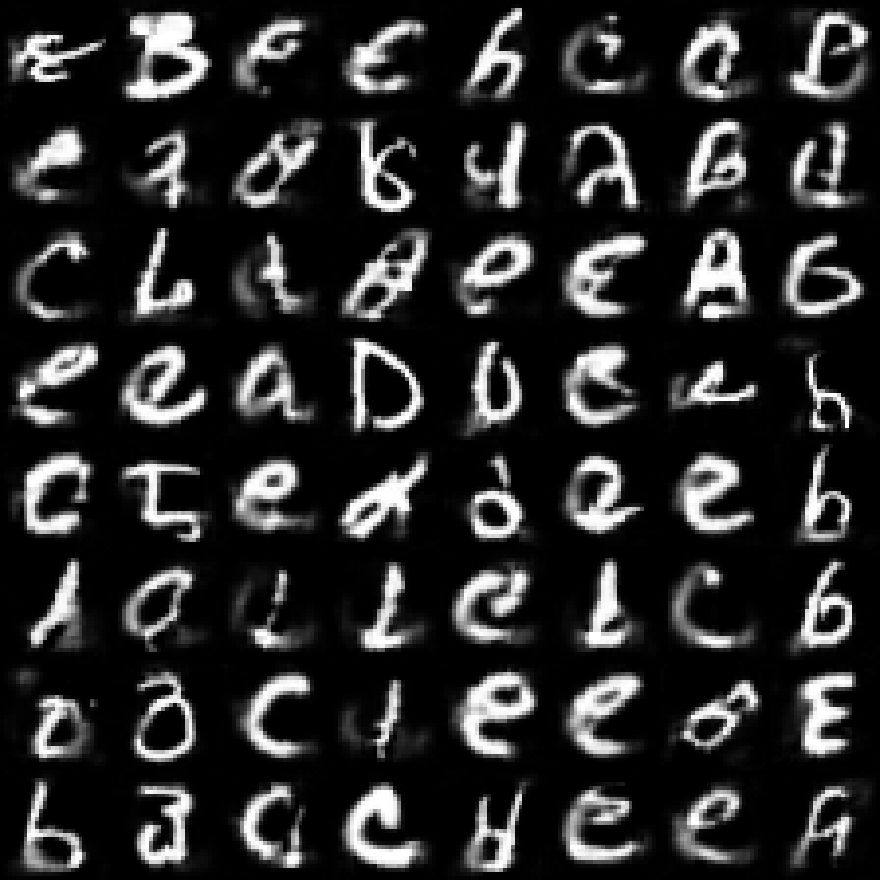}
\caption{Reverse sampling from a  VDT policy from MNIST to EMNIST.}\label{fig:more-unpaired2}
\end{figure}

\paragraph{Conditional generation of MNIST digits} \label{app:more-mnist-conditional}
We apply the CFG procedure explained in Appendix~\ref{app:implementation-details} to our setting for generating MNIST 
images augmented with their class labels. Results for various scale parameters are shown on 
Figure~\ref{fig:more-conditional}. The results indicate that CFG affects the generation 
quality in the same way as it does for diffusion models: we obtain enhanced image quality at the expense of smaller 
sample diversity.

\begin{figure}
\center
\begin{centering}
\includegraphics[width=.45\linewidth]{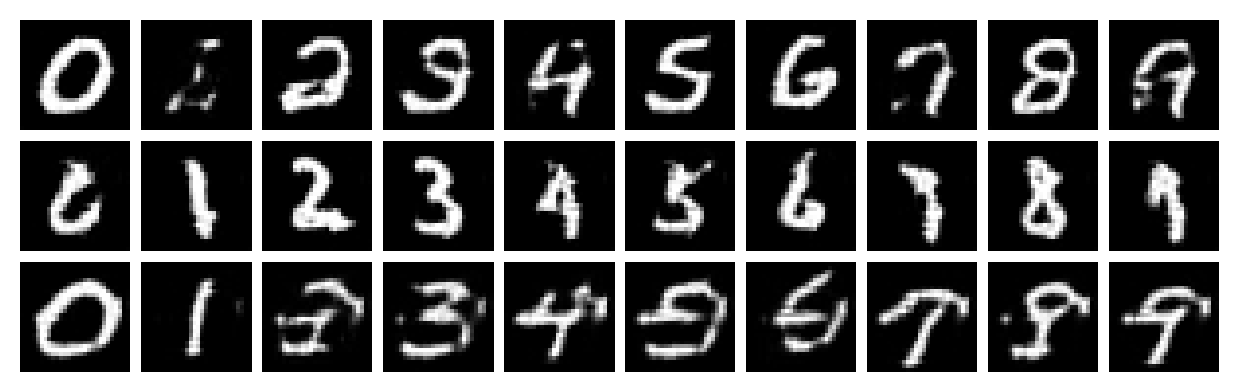}
\includegraphics[width=.45\linewidth]{
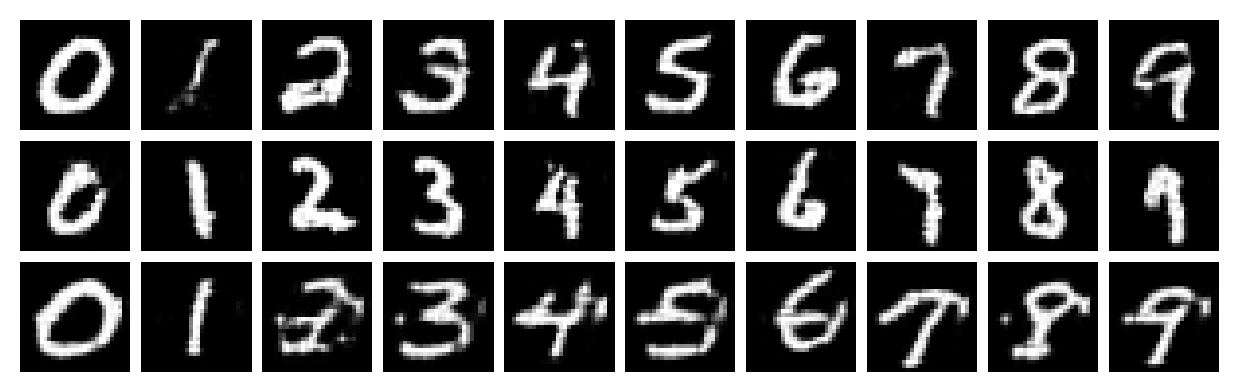}
\end{centering}

\begin{centering}
\includegraphics[width=.45\linewidth]{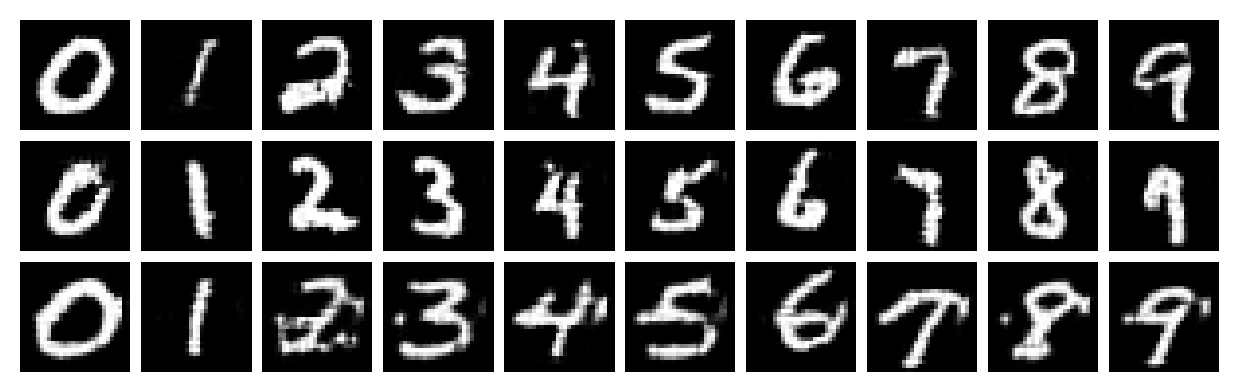}
\includegraphics[width=.45\linewidth]{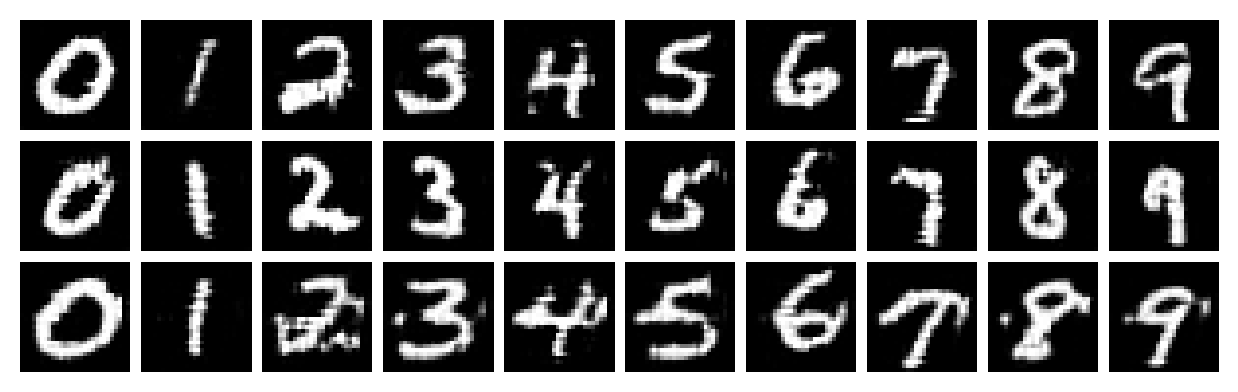}
\end{centering}
\caption{MNIST digits generated by CFG with various guidance scales: $\alpha\in\ev{1.0, 1.5, 2.0, 3.0}$.}\label{fig:more-conditional}
\end{figure}

\end{document}